%% file: main.tex
\documentclass[11pt]{article}
\usepackage{graphicx}
\usepackage{amsmath,amsthm,amscd}
\usepackage{amssymb}
\usepackage{fancybox}
\usepackage{latexsym}
\usepackage{fancyhdr}
\usepackage{lastpage}
\usepackage{slashbox}
\usepackage{url}
\usepackage{lscape}
\usepackage{multirow}
\usepackage{mlapa}

\input setup

\newtheorem{theorem}{Theorem}
\newtheorem{proposition}{Proposition}
\newtheorem{definition}{Definition}
\newtheorem{lemma}{Lemma}
\newtheorem{remark}{Remark}
\newtheorem{example}{Example}

\begin{document}
\title{Condition Number Analysis\\ of Kernel-based Density Ratio Estimation 
}
\author{
  Takafumi Kanamori\\ Nagoya University \\ \tt{kanamori@is.nagoya-u.ac.jp}
  \and
 Taiji Suzuki\\ University of Tokyo\\ \tt{s-taiji@stat.t.u-tokyo.ac.jp}
 \and
 Masashi Sugiyama\\ Tokyo Institute of Technology\\ \tt{sugi@cs.titech.ac.jp}
 }
%
%
%
%
\date{}
\maketitle

\begin{abstract}
 The ratio of two probability densities can be used for solving various machine learning tasks
 such as covariate shift adaptation (importance sampling),
 outlier detection (likelihood-ratio test), and feature selection
 (mutual information).
 Recently, several methods of directly estimating the density ratio 
 have been developed, e.g., kernel mean matching,
 maximum likelihood density ratio estimation,
 and least-squares density ratio fitting.
 In this paper, we consider a kernelized variant of 
 the least-squares method and investigate its theoretical properties
 from the viewpoint of the condition number using smoothed analysis
 techniques---the condition number
 of the Hessian matrix determines the convergence rate of optimization
 and the numerical stability. 
 We show that the kernel least-squares method has a smaller condition number
 than a version of kernel mean matching and other M-estimators, 
 implying that the kernel least-squares method has preferable numerical properties. 
 We further give an alternative formulation of the kernel least-squares estimator
 which is shown to possess an even smaller condition number. 
 We show that numerical studies meet our theoretical analysis. 
\end{abstract}

\section{Introduction}\label{sec:Introduction}


The problem of estimating the ratio of two probability densities
is attracting a great deal of attention these days, 
since the density ratio can be used for various purposes
such as covariate shift adaptation 
\cite{JSPI:Shimodaira:2000,ICML:Zadrozny:2004,StatDeci:Sugiyama+Mueller:2005,NIPS2006_915,JMLR:Sugiyama+etal:2007,bickel09:_discr_learn_under_covar_shift}, 
outlier detection \cite{nc:schoelkopf+platt+shawe-taylor:2001,mach:Tax+Duin:2004,AIR:Hodge+Austin:2004,ICDM:Hido+etal:2008},
and divergence estimation \cite{NIPS:Nguyen+etal:2008,FSDM:Suzuki+etal:2008}.

A naive approach to density ratio estimation is to first separately estimate two
probability densities
and then take the ratio of the estimated densities.
However, density estimation is known to be a hard problem particularly in 
high-dimensional cases unless we have simple and good parametric density models
\cite{book:Vapnik:1998,book:Haerdle+etal:2004},
which may not be the case in practice.

Recently, methods of directly estimating the density ratio
without going through density estimation have been developed.
The \emph{kernel mean matching} (KMM) method \cite{NIPS2006_915} directly gives
estimates of the density ratio by matching the two distributions efficiently
using a special property of \emph{universal reproducing kernel Hilbert spaces (RKHSs)}
\cite{JMLR:Steinwart:2001}. 
Another approach is an M-estimator 
\cite{NIPS:Nguyen+etal:2008} based on
non-asymptotic variational characterization
of the $f$-divergence \cite{JRSS-B:Ali+Silvey:1966,SSM-Hungary:Csiszar:1967}.
See also \emcite{NIPS:Sugiyama+etal:2008} for a similar
algorithm under the Kullback-Leibler divergence.
Non-parametric convergence properties of the M-estimator in RKHSs
have been elucidated under the Kullback-Leibler divergence
\cite{NIPS:Nguyen+etal:2008,AISM:Sugiyama+etal:2008}. 
A squared-loss version of the M-estimator for linear density-ratio models
called \emph{unconstraint Least-Square Importance Fitting} (uLSIF) has been developed
and has been shown to possess useful computational properties,
e.g., a closed-form solution is available
and the leave-one-out cross-validation score can be analytically computed
\cite{kanamori09:_least_squar_approac_to_direc_impor_estim}. 
 


In this paper, we consider a kernelized variant of uLSIF (KuLSIF)
and analyze its properties in numerical optimization
from the viewpoint of the {\em condition number}. 
The condition number of the Hessian matrix of objective function 
plays a crucial role 
\cite{luenberger08:_linear_and_nonlin_progr,bertsekas96:_nonlin_progr},
i.e., it determines the convergence rate of optimization and the numerical stability. 
When an objective function to be optimized is randomly chosen and fed into an optimization
algorithm, the computational cost of an algorithm can be assessed by the distribution of
the condition number. 
The distribution of condition numbers of randomly perturbed matrices has been studied 
by the name of \emph{smoothed analysis}
\cite{spielman04:_smoot_analy_of_algor,sankar06:_smoot_analy_of_condit_number}. 
Smoothed analysis was originally introduced to explain the
success of algorithms and heuristics that could not be well-understood through traditional
worst-case and average-case analysis---it gives a more realistic analysis of the practical
performance of algorithms. 

We apply smoothed analysis techniques to derive the distribution of the condition
number of density-ratio estimation algorithms.
More specifically, we first give a unified view of
the objective functions of KuLSIF and KMM.
Then we show that
KuLSIF has a smaller condition number than an ``induction'' variant of KMM,
implying that KuLSIF is  more preferable than KMM in optimization.
We further show that KuLSIF---which could be regarded as an instance of
M-estimators---has the smallest condition number among all M-estimators
in the min-max sense (i.e., the worst condition number over all
density ratio functions is the smallest in KuLSIF).
We also give probabilistic evaluation of the condition number of M-estimators
and show that KuLSIF is favorable. These theoretical findings are also verified through
numerical experiments. 
We further give an alternative formulation of KuLSIF which is denoted as Reduced-KuLSIF, 
and show that it possesses an even smaller condition number.

The rest of this paper is organized as follows.
In Section~\ref{sec:Estimation_of_Density_Ratio},
we formulate the problem of density ratio estimation
and briefly review existing methods.
In Section~\ref{sec:Kernel-uLSIF},
we describe the KuLSIF algorithm, and show its fundamental properties 
such as the convergence rate and availability of 
the analytic-form solution and the analytic-form leave-one-out cross-validation score.
Section~\ref{sec:condition-number} is the main contribution of this paper,
giving condition number analysis of density ratio estimation methods.
In Section~\ref{sec:reduction_cond_num}, we give an alternative
formulation of KuLSIF by transforming loss functions
and show that is possesses an even smaller condition number. 
In Section~\ref{sec:Simulation_Results}, 
we experimentally investigate the behavior of the condition numbers,
confirming validity of our theories.
In Section~\ref{sec:Conclusion}, we conclude
by summarizing our contributions and showing possible future directions.

\section{Estimation of Density Ratio}\label{sec:Estimation_of_Density_Ratio}
We formulate the problem of density ratio estimation
and briefly review existing methods.

\subsection{Formulation and Notations}
Consider two probability distributions $P$ and $Q$ on a probability space $\ZC$. 
Assume that both distributions have the probability densities $p$ and $q$, respectively. 
We assume $p(x)>0$ for all $x\in\ZC$. Suppose that we are given two sets of independent
and identically distributed (i.i.d.) samples, 
\begin{equation}
 X_1,\ldots,X_n  \iid  P,\qquad
 Y_1,\ldots,Y_m  \iid  Q. 
 \label{eqn:samples}
\end{equation}
Our goal is to estimate the density ratio 
\[
w_0(x)=\frac{q(x)}{p(x)}\;(\ge0)
\]
based on the observed samples.

We summarize some notations to be used throughout the paper. 
For a vector $a$ in the Euclidean space, $\|a\|$ denotes the Euclidean norm. 
Given a probability distribution $P$ and a random variable $h(X)$, we denote the
expectation of $h(X)$ under $P$ by $\int h dP$ or $\int h(x)P(dx)$. 
Given samples $X_1,\ldots,X_n$ from $P$, the empirical distribution
is denoted by $P_n$. The expectation $\int hdP_n$ denotes the empirical means of $h(X)$,
that is, $\frac{1}{n}\sum_{i=1}^{n}h(X_i)$. 
Let $\|\cdot\|_{\infty}$ be the infinity norm, and $\|\cdot\|_{P}$ be 
the $L_2$-norm under the probability $P$, i.e. $\|h\|_{P}^2=\int|h|^2dP$. 
For a reproducing kernel Hilbert space (RKHS) $\mathcal{H}$ 
\cite{book:Schoelkopf+Smola:2002}, the inner product and the norm on $\mathcal{H}$ are
denoted as $\<\cdot,\cdot\>_{\mathcal{H}}$ and $\|\cdot\|_{\mathcal{H}}$, respectively. 

Below we review several approaches to density ratio estimation.

\subsection{Kernel Mean Matching} \label{sec:KMM} 
The {\em kernel mean matching} (KMM) method allows us to directly obtain an estimate of
$w_0(x)$ at $X_1,\ldots,X_n$ without going through density 
estimation \cite{NIPS2006_915}.

The basic idea of KMM is to find $w_0(x)$ such that the mean discrepancy between
non-linearly transformed samples drawn from $P$ and $Q$ is minimized in a 
{\em universal reproducing kernel Hilbert space} \cite{JMLR:Steinwart:2001}. 
We introduce the definition of universal kernel below. 
\begin{definition}
 [\emcite{JMLR:Steinwart:2001}]
 \label{definition:universal-kernel}
 A continuous kernel $k$ on a compact metric space $\mathcal{Z}$ is called universal if
 the RKHS $\mathcal{H}$ of $k$ is dense in the set of all continuous functions on
 $\mathcal{Z}$, that is, for every continuous function $g$ on $\mathcal{Z}$ and all
 $\varepsilon>0$, there exists an $f\in\mathcal{H}$ such that
 $\|f-g\|_\infty<\varepsilon$. The corresponding RKHS is called universal RKHS. 
\end{definition}
The Gaussian kernel is an example of universal kernels. 
Let $\mathcal{H}$ be a universal RKHS endowed with the kernel function 
$k:\ZC\times \ZC\longrightarrow\Real$.
For any $x\in\ZC$, the function $k(\cdot,x)$ is regarded as an element of $\mathcal{H}$. 
Then, it has been shown that the solution of the following optimization problem agrees 
with the true density ratio $w_0$: 
\begin{align*}
  \min_{w}
  \;\;  \frac{1}{2}\,\bigg\|\int w(x)k(\cdot,x) P(dx)-\int k(\cdot,y) Q(dy) \bigg\|^2_{\cal H},\qquad 
 \mbox{s.t. }
  \int\! w dP=1
  \ \ \mbox{and} \ \ w\ge0. 
\end{align*}
Indeed, when $w=w_0$, the loss function equals to zero. 
An empirical version of the above problem is reduced to the following convex quadratic
program: 
\begin{align}
 \begin{array}{l}
  \displaystyle  \min_{w_1,\ldots,w_n}
  \frac{1}{2n}\sum_{i,j=1}^{n} w_i w_{j}k(X_i,X_j)
  -\frac{1}{m}\sum_{j=1}^{m}\sum_{i=1}^{n} w_ik(X_i,Y_j),\vspace*{2mm}\\
  \displaystyle \qquad \mbox{s.t.\ }
 \bigg|\frac{1}{n}\sum_{i=1}^{n}w_i-1\bigg| \leq \epsilon
  \ \ \mbox{and} \ \ 
 0\leq w_1, w_2,\ldots, w_n\leq B. 
 \end{array}
 \label{eqn:kmm-opt-prob}
\end{align}
Tuning parameters, $B \geq 0$ and $\epsilon\geq 0$, 
control the regularization effects. 
The solution 
$\widehat{w}_1,\ldots,\widehat{w}_n$ is an estimate of the density ratio at the samples
from $P$, i.e., $w_0(X_1),\ldots,w_0(X_n)$. Note that KMM does not estimate the function
$w_0$ on $\ZC$ but the values on sample points (i.e., transduction). 

\subsection{M-estimator based on $f$-divergence Approach}
An estimator of the density ratio based on
the $f$-divergence \cite{JRSS-B:Ali+Silvey:1966,SSM-Hungary:Csiszar:1967}
has been proposed by \emcite{NIPS:Nguyen+etal:2008}. 
Let $\varphi:\Real\rightarrow\Real$ be a convex function, then the
$f$-divergence between $P$ and $Q$ is defined by the integral 
\begin{align*}
 I(P,Q)~=~\int \varphi(q/p) dP. 
\end{align*}
Setting $\varphi(z)=-\log z$, we obtain the Kullback-Leibler divergence as an example of
$f$-divergences. 
Let the conjugate dual function $\psi$ of $\varphi$ be 
\[
\psi(z)=\sup_{u\in\Real}\{zu-\varphi(u)\}=-\inf_{u\in\Real}\{\varphi(u)-zu\}. 
\]
When $\varphi$ is a convex function, we also have
\begin{equation}
\varphi(z)=-\inf_{u\in\Real}\{\psi(u)-zu\}. 
 \label{eqn:convex-duality}
\end{equation}
Substituting \eqref{eqn:convex-duality} into the $f$-divergence, 
we obtain another expression,  
\begin{align}
 I(P,Q)~=~-\inf_w \left[\int\!\!\psi(w)dP-\int\!\!wdQ\right], 
 \label{eqn:f-div-conjugate-dual}
\end{align}
where the infimum is taken over all measurable functions $w:\ZC\rightarrow\Real$. The
infimum is attained at the function $w$ such that 
\begin{align*}
\frac{q(x)}{p(x)}= \psi'(w(x)), 
 \end{align*}
where $\psi'$ is the derivative of $\psi$. Approximating 
\eqref{eqn:f-div-conjugate-dual} with the empirical distributions $P_n$ and $Q_m$, we
obtain the empirical loss function. This estimator is referred to as the 
{\em M-estimator} of the density ratio.  
A more practical algorithm for the Kullback-Leibler divergence has been independently
proposed in \emcite{NIPS:Sugiyama+etal:2008}. 

When an RKHS $\mathcal{H}$ is employed as a statistical model, an estimator is obtained
by minimizing the loss function which approximates \eqref{eqn:f-div-conjugate-dual} over 
$\mathcal{H}$, 
\begin{align}
 \inf_w\ \int \psi(w) dP_n-\int w dQ_m+\frac{\lambda}{2}\|w\|_{\mathcal{H}}^2, 
 \quad w\in\mathcal{H}. 
 \label{eqn:f-div-RKHS-estimator}
\end{align}
The density ratio $w_0$ is estimated by $\psi'(\widehat{w}(x))$, where $\widehat{w}$ is the
minimizer of 
\eqref{eqn:f-div-RKHS-estimator}. 
The regularization term $\frac{\lambda}{2}\|w\|_{\mathcal{H}}^2$ with the
regularization parameter $\lambda$ is introduced to avoid overfitting. 
In the RKHS $\mathcal{H}$, the representer theorem \cite{JMAA:Kimeldorf+Wahba:1971} 
is applicable, and the optimization problem on $\mathcal{H}$ 
is reduced to a finite dimensional optimization problem. Statistical convergence
properties of the kernel estimator for the Kullback-Leibler divergence have been
investigated in \emcite{NIPS:Nguyen+etal:2008} and \emcite{AISM:Sugiyama+etal:2008}.

\subsection{Least-squares Approach}
\label{subsec:Least-squares Approach}
The linear model 
\begin{equation}
 \widehat{w}(x)~=~\sum_{i=1}^{b}\alpha_i h_i(x)
  \label{eqn:linear-model-density-ratio}
\end{equation}
is assumed for estimation of the density ratio $w_0$, where the coefficients
$\alpha_1,\ldots,\alpha_b$ are the parameters of the model. The basis functions 
$h_i,\ i=1,\ldots,b$ are chosen so that the non-negativity condition
$h_i(x)\geq 0$ is satisfied. A practical choice would be the Gaussian kernel function
$h_i(x)=e^{-\|x-c_i\|^2/2\sigma^2}$ with appropriate kernel center $c_i\in\ZC$ and
kernel width $\sigma$ \cite{NIPS:Sugiyama+etal:2008}. 

The {\em unconstraint least-square importance fitting} (uLSIF)
\cite{kanamori09:_least_squar_approac_to_direc_impor_estim}
estimates the parameter $\alpha$
based on the square error: 
\begin{align*}
\frac{1}{2}\int(\widehat{w}-w_0)^2 dP
 ~=~&\frac{1}{2}\int \widehat{w}^2 dP - \int\widehat{w}dQ + \frac{1}{2}\int w_0^2dP. 
\end{align*}
The last term in the above expression is a constant and can be safely
ignored when minimizing the square error of the estimator $\widehat{w}$. 
Therefore, the solution of the following minimization problem over the linear model, 
\begin{align}
 \min_w\ \frac{1}{2}\int w^2 dP_n - \int w dQ_m
 +\lambda\cdot\reg(\alpha) ,
 \label{eqn:population-mean-uLSIF}
\end{align}
is expected to approximate the true density ratio $w_0$, 
where the regularization term $\reg(\alpha)$ with the regularization
parameter $\lambda$ is introduced to avoid overfitting. 
We define the column vector $\alpha=(\alpha_1,\ldots,\alpha_b)^\top$ and the vector-valued 
function $h(x)=(h_1(x),\ldots,h_b(x))^\top$. 
Substituting the linear model \eqref{eqn:linear-model-density-ratio} into the objective
function of \eqref{eqn:population-mean-uLSIF},
we obtain 
\begin{align}
 \min_{\alpha\in\Real^b}\ 
 \frac{1}{2}\alpha^\top \widehat{H} \alpha - \widehat{g}^\top \alpha 
 +\lambda\cdot\reg(\alpha) 
 \label{uLSIF-loss}, 
\end{align}
where $\widehat{H}$ and $\widehat{g}$ are
the $b$ by $b$ matrix and the $b$-dimensional vector
defined as $\widehat{H}=\int h h^\top dP_n$ and $\widehat{g}=\int h dQ_m$, 
respectively. 
Let $\widehat{\alpha}$ be the minimizer of \eqref{uLSIF-loss}, then the estimator
of $w_0$ is given as $\widehat{w}(x)=\sum_{i=1}^{b}\widehat{\alpha}_i h_i(x)$. 
There are several ways to impose the non-negativity
condition $\widehat{w}(x)\geq 0$ 
\cite{kanamori09:_least_squar_approac_to_direc_impor_estim}. 
Here,
truncation of $\widehat{w}$ defined as 
\[
\widehat{w}_+(x)=\max\{\widehat{w}(x),\ 0\}
\]
is used for obtaining a non-negative estimator. 

Note that the loss function \eqref{eqn:f-div-conjugate-dual} with $\psi(z)=z^2/2$ is
essentially equivalent to the loss of uLSIF. 
uLSIF has an advantage in computation over other M-estimators: 
When $\reg(\alpha)=\|\alpha\|^2/2$, 
the estimator $\widehat{\alpha}$ can be obtained in an analytic form. 
As a result, the leave-one-out cross-validation (LOOCV) score can also be computed in a
closed form \cite{kanamori09:_least_squar_approac_to_direc_impor_estim}, 
which allows us to compute the LOOCV score very efficiently.
LOOCV is an (almost) unbiased estimator of the prediction error and
can be used for determining hyper-parameters such as regularization parameter $\lambda$ or
Gaussian kernel width $\sigma$.

\section{Kernel uLSIF}\label{sec:Kernel-uLSIF}
The purpose of this paper is to show that
a kernelized variant of uLSIF (which we refer to as \emph{kernel uLSIF}; KuLSIF)
has good theoretical properties and thus useful.
In this section, we formalize the KuLSIF algorithm and
briefly show its fundamental properties.
Then in the next section, we analyze the computational efficiency of KuLSIF algorithm from
the viewpoint of the condition number. 


\subsection{uLSIF on RKHS}\label{sec:uLSIF_on_RHKS}
We assume that the model for the density ratio is an RKHS $\mathcal{H}$
endowed with a kernel function $k$ on $\ZC\times \ZC$, and we consider the optimization
problem \eqref{eqn:population-mean-uLSIF} on $\mathcal{H}$. 
According to \eqref{eqn:population-mean-uLSIF}, the estimator $\widehat{w}$ is obtained as
\begin{align}
 \label{eqn:kernel-uLSIF-minprob}
 \begin{array}{l}
  \displaystyle
   \min_w\ \frac{1}{2}\int w^2 dP_n-\int wdQ_m + \frac{\lambda}{2} \|w\|_{\cal H}^2,
   \quad \st\ \  w\in \mathcal{H}. 
 \end{array}
\end{align}
The regularization term $\frac{\lambda}{2}\|w\|_{\cal H}^2$ with the regularization
parameter $\lambda$ $(\geq 0)$ is introduced to avoid overfitting. The truncated estimator
$\widehat{w}_+=\max\{\widehat{w},0\}$ may be preferable in practice;
the estimation procedure of $\widehat{w}$ or $\widehat{w}_+$ based on
\eqref{eqn:kernel-uLSIF-minprob} is called KuLSIF. 


The following theorem reveals the convergence rate of the estimators $\widehat{w}$ and
$\widehat{w}_+$. 
\begin{theorem}[Convergence Rate of KuLSIF]
 \label{theorem:non-para-bound}
 Assume that the domain $\ZC$ is compact. Let $\HC$ be an RKHS with the Gaussian kernel. 
 Suppose that $q/p=w_0\in{\cal H}$, 
 and $\|w_0\|_{\mathcal{H}}<\infty$. 
 Set the regularization parameter $\lambda=\lambda_{n,m}$ so that
 \begin{align*}
  \lim_{n,m\rightarrow\infty}\lambda_{n,m}=0,\qquad
  \lambda_{n,m}^{-1} = O((n\wedge m)^{1-\delta}), 
 \end{align*}
 where $n \wedge m=\min\{n,m\}$ and $\delta$ is arbitrary number satisfying $0<\delta<1$. 
 Then the estimators $\what$ and $\what_+$ satisfy
 \begin{align*}
  \|\what_+-w_0\|_{P} ~\leq~  \|\what-w_0\|_{P} ~=~ O_p(\lambda_{n,m}^{1/2}), 
 \end{align*}
 where $\|\cdot\|_{P}$ is the $L_2$-norm under the probability $P$. 
\end{theorem}
Proofs may be found in Appendix~\ref{appendix:Proof_Convergence_Theorem}. 
By choosing small $\delta>0$, the convergence rate will get close to the order of
$O(1/\sqrt{n\wedge m})$ which is the convergence rate for parametric models. 
See \emcite{NIPS:Nguyen+etal:2008} and \emcite{AISM:Sugiyama+etal:2008}
for similar convergence analysis under the Kullback-Leibler divergence.
\begin{remark}
 \label{remark:convergence-rate-extension-kernel}
 Although Theorem \ref{theorem:non-para-bound} focuses on the Gaussian kernel, extension
 to the other kernels is straightforward. 
 Let $\mathcal{Z}$ be a probability space, and $k$ be a kernel function over 
 $\mathcal{Z}\times \mathcal{Z}$, and suppose $\sup_{x\in\mathcal{Z}}k(x,x)<\infty$. 
 According to the proof of Theorem \ref{theorem:non-para-bound}, 
 we assume that the bracketing entropy $H_B(\delta,\mathcal{H}_M,P)$
 is bounded above by $O(M/\delta)^\gamma$, where $0<\gamma<2$ (see the proof in 
 Appendix~\ref{appendix:Proof_Convergence_Theorem} for the definition). Then, we obtain   
 \begin{align*}
  \|\what_+-w_0\|_{P} ~\leq~ \|\what-w_0\|_{P} ~=~ O_p(\lambda_{n,m}^{1/2}),
 \end{align*}
 where $\lambda_{n,m}^{-1}=O((n\wedge m)^{1-\delta})$ with $1-2/(2+\gamma)<\delta<1$. 
\end{remark}


\subsection{Analytic-form Solution of KuLSIF}\label{sec:Computation_of_kernel-uLSIF}
The problem \eqref{eqn:kernel-uLSIF-minprob} is an infinite dimensional 
optimization problem, if the dimension of ${\cal H}$ is infinite. The representer theorem 
\cite{JMAA:Kimeldorf+Wahba:1971}, however, is applicable to RKHSs, and then,
we immediately have the following theorem. 
\begin{theorem}
  \label{thm:representer-theoerm}
 Suppose the samples \eqref{eqn:samples} are observed. 
 The estimator $\widehat{w}$ given as the solution of \eqref{eqn:kernel-uLSIF-minprob} has
 the form of 
 \begin{align}
  \widehat{w}(z)~=~\sum_{i=1}^n\alpha_i k(z,X_i)+\sum_{j=1}^m\beta_j k(z,Y_j), 
  \label{eqn:representer-theorem}
 \end{align}
 where $\alpha_1,\ldots,\alpha_n, \beta_1,\ldots,\beta_m\in\Real$. 
\end{theorem}
The theorem follows a direct application of the original representer theorem, so we omit
its proof. This theorem shows that the estimator $\widehat{w}$ lies in a finite
dimensional subspace of ${\cal H}$.


Furthermore, for KuLSIF (i.e., the squared-loss),
the parameters in $\widehat{w}(z)$ can be obtained
analytically. 
Let $K_{11}$, $K_{12}$, $K_{21}$, and $K_{22}$ be the sub-matrices of the Gram matrix:
\begin{align*}
 (K_{11})_{ii'}~=~&k(X_i,X_{i'}),\;\;\;\; 
 (K_{12})_{ij}~=~k(X_i,Y_j),\;\;\;\; 
 K_{21}~=~K_{12}^\top,\;\;\;\; (K_{22})_{jj'}~=~k(Y_j,Y_{j'}),
\end{align*}
where $i,i'=1,\ldots,n,\ j,j'=1,\ldots,m$. 
Let $\1_m=(1,\ldots,1)^\top\in\Real^m$ for positive integer $m$. 
Then the estimated parameters $\alpha_i$ and $\beta_j$ are given as follows. 
\begin{theorem}[Analytic Solution of KuLSIF]
 \label{thm:beta-const}
 Suppose that 
 the regularization parameter $\lambda$ is strictly positive. 
 Then the estimated parameters in KuLSIF are given as 
 \begin{align}
  \alpha~=~(\alpha_1,\ldots,\alpha_n)^\top
  ~=~& -\frac{1}{m\lambda}\left(K_{11}+n\lambda I_n\right)^{-1}K_{12}\1_m,
  \label{alpah-b}\\
  \beta~=~(\beta_1,\ldots,\beta_m)^\top  
  ~=~&\frac{1}{m\lambda}\1_m,
  \label{beta}
 \end{align}
where $I_n$ is the $n$ by $n$ identity matrix. 
\end{theorem}

\begin{proof}
We start to prove the theorem for general M-estimator based on $f$-divergences. 
We consider the minimization problem of the loss function 
 \begin{align*}
  \int \psi(w) dP_n-\int w dQ_m +\frac{\lambda}{2}\|w\|_{\mathcal{H}}^2
 \end{align*}
subject to 
\[
 w=\sum_{j=1}^{n}\alpha_j k(\cdot,X_j)+\sum_{\ell=1}^{m}\beta_\ell k(\cdot,Y_\ell).
\]
Suppose $\psi$ is a differentiable convex function. 
Let $v(\alpha,\beta)\in\Real^n$ be a vector-valued function defined as 
\[
 v(\alpha,\beta)_i =
 \psi' \left( \sum_{j=1}^{n}\alpha_j k(X_i,X_j)
 +\sum_{\ell=1}^{m}\beta_\ell k(X_i,Y_\ell) \right) ,\quad i=1,\ldots,n,
\]
where $\psi'$ denotes the derivative of $\psi$.
Then, the extremal condition of the loss function is given as
\begin{align*}
 &\frac{1}{n}K_{11}v(\alpha,\beta)-\frac{1}{m}K_{12}\1_m+\lambda K_{11}\alpha +\lambda
 K_{12}\beta~=~0,\quad \text{and}\\
 &\frac{1}{n}K_{21}v(\alpha,\beta)-\frac{1}{m}K_{22}\1_m+\lambda K_{22}\beta +\lambda
 K_{21}\alpha~=~0. 
\end{align*}
 If $\alpha$ and $\beta$ satisfy the above conditions,
 they are the optimal solution because the loss function is convex 
 in $\alpha$ and $\beta$. Substituting $\beta=\frac{1}{m\lambda}\1_m$, we obtain 
\begin{align*}
 &\frac{1}{n}K_{11}v(\alpha,\1_m/m\lambda)+\lambda K_{11}\alpha=0,\quad\text{and}\\ 
 &\frac{1}{n}K_{21}v(\alpha,\1_m/m\lambda)+\lambda K_{21}\alpha=0. 
\end{align*}
Hence, if the equation 
\begin{align}
 \frac{1}{n}v(\alpha,\1_m/m\lambda)+\lambda \alpha=0
 \label{eqn:beta_sol_proof}
\end{align}
has a solution, it is revealed that $\beta=\frac{1}{m\lambda}\1_m$ is a part of
the optimal solution. For $\psi(z)=z^2/2$, we have
\[
 v(\alpha,\beta) = K_{11}\alpha+K_{12}\beta, 
\]
thus, \eqref{eqn:beta_sol_proof} is reduced to
\begin{align}
 \left(K_{11}+n\lambda I_n\right)\alpha = -\frac{1}{m\lambda}K_{12}\1_m.
 \label{eqn:uLSIF-linear-eq}
\end{align}
The coefficient matrix is non-singular. Therefore, the estimator is represented by
 \eqref{alpah-b} and \eqref{beta}.  
\end{proof}

\begin{remark}
\label{rem:beta-const} 
 As shown in the proof of Theorem~\ref{thm:beta-const}, 
 the estimate $\beta$ for any $f$-divergence
 \eqref{eqn:f-div-RKHS-estimator}
 is given as \eqref{beta} (but not \eqref{alpah-b}) 
 under the condition that Equation \eqref{eqn:beta_sol_proof} has a solution with respect
 to $\alpha$. 
\end{remark}
Eventually, the estimator based on the $f$-divergence is given by solving the following
optimization problem, 
\begin{align}
 \begin{array}{l}
  \displaystyle
  \inf_w\ \int \psi(w) dP_n-\int w dQ_m+\frac{\lambda}{2}\|w\|_{\mathcal{H}}^2,  \\
  \displaystyle	
  \quad \st\ \
  w(\cdot)=\sum_{i=1}^{n}\alpha_ik(\cdot,X_i)+\frac{1}{m\lambda}\sum_{j=1}^{m}k(\cdot,Y_j),\
  \ \ \alpha_1,\ldots,\alpha_n\in\Real. 
 \end{array}
 \label{eqn:f-div-RKHS-est-representer}
\end{align} 
When $\psi(z)=z^2/2$, the problem \eqref{eqn:f-div-RKHS-est-representer} is reduced to 
\begin{align}
 \min_\alpha\ \frac{1}{2}\alpha^\top\left(\frac{1}{n}K_{11}^2+\lambda K_{11}\right)\alpha
 +\frac{1}{nm\lambda}\1_m^\top K_{21}K_{11}\alpha,\quad \alpha\in \Real^n
 \label{eqn:KuLSIF-opt-prob}
\end{align}
by ignoring the term independent of the parameter $\alpha$. 
On the other hand, Theorem~\ref{thm:beta-const} guarantees that the parameter $\alpha$ in
KuLSIF is obtained by the optimal solution of the following optimization problem 
\footnote{
We used the fact that the solution of $Ax=b$ is given as the minimizer
of $\frac{1}{2}x^\top Ax-b^\top x$, when $A$ is positive-semidefinite. }: 
\begin{align}
 \min_\alpha\ \frac{1}{2}\alpha^\top \left(\frac{1}{n}K_{11}+\lambda I_n\right)
 \alpha+\frac{1}{nm\lambda} \1_m^\top K_{21}\alpha,\quad \alpha\in \Real^n. 
 \label{eqn:reduced-KuLSIF}
\end{align}
The estimator given by solving the optimization problem
\eqref{eqn:reduced-KuLSIF} is denoted as \emph{Reduced-KuLSIF} (R-KuLSIF). 
Although KuLSIF and R-KuLSIF share the same optimal
solution, the loss function is different. In a later section, we make clear that R-KuLSIF
is more preferable than the other estimators including KuLSIF from the viewpoint of
numerical computation, especially when the sample size is large.

\subsection{Leave-one-out Cross-validation}\label{sec:Leave-One-Out_Cross_Validation}
In addition to the solutions $\alpha_i$ 
and $\beta_j$, 
the leave-one-out cross-validation (LOOCV) score can also be obtained
analytically in KuLSIF.
The accuracy of the KuLSIF estimator $w_+=\max\{w,0\}$
is measured by $\frac{1}{2}\int w_+^2 dP-\int w_+ dQ$,
which is equal to the square error of $w_+$ up to a constant term.
Then the LOOCV score of $w_+$ under
the square error is defined as 
\begin{align}
 {\rm LOOCV}~=~\frac{1}{n\wedge m}\sum_{\ell=1}^{n \wedge m}
 \left\{
 \frac{1}{2}(\widehat{w}_{+}^{(\ell)}(x_\ell))^2-\widehat{w}_+^{(\ell)}(y_\ell)\right\}, 
 \label{eqn:LOOCV}
\end{align}
where 
$\widehat{w}_+^{(\ell)}=\max\{\widehat{w}^{(\ell)},0\}$ is the estimator based on the
samples except $x_\ell$ and $y_\ell$. The index of removed samples
could be different, for example $x_{\ell_1}$ and $y_{\ell_2}$, but 
for the sake of simplicity, we suppose that the samples $x_\ell$ and $y_\ell$ are removed
in the computation of LOOCV. Hyper-parameters achieving the minimum value of LOOCV will be
a good choice. 

Thanks to the analytic solutions \eqref{alpah-b} and \eqref{beta},
the leave-one-out solution $\widehat{w}^{(\ell)}$ can be computed
efficiently from $\widehat{w}$ 
by the use of the Sherman-Woodbury-Morrison formula \cite{book:Golub+vanLoan:1996}.
The detail of the analytic LOOCV expression is deferred to Appendix~\ref{appendix:LOOCV}---the
derivation follows 
a similar line to 
\cite{kanamori09:_least_squar_approac_to_direc_impor_estim}
which deals with a linear model \eqref{eqn:representer-theorem};
a minor difference is that
removing the sample $(x_\ell,y_\ell)$ in KuLSIF changes the
basis functions due to the kernel expression.

\section{Relation between KuLSIF and KMM}
\label{sec:Relation_KMM}
We show the relation between KuLSIF and KMM. 

We assume that the true density ratio $w_0=q/p$ is included in ${\cal H}$. 
As shown in Section~\ref{sec:Estimation_of_Density_Ratio}, the loss function of KMM on
${\cal H}$ is defined as 
\begin{align*}
  L_{\rm KMM}(w)~=~&\frac{1}{2}\|\Phi(w)\|^2_{\cal H},\\ 
 \Phi(w) ~=~& \int\! k(\cdot,x) w(x) P(dx)-\int\! k(\cdot,y) Q(dy). 
\end{align*}
In the estimation phase, an empirical approximation of $L_{\rm KMM}$ is optimized in the KMM
algorithm. 
On the other hand, the (unregularized) loss function of KuLSIF is given by
\begin{align*}
 L_{\rm KuLSIF}(w)~=~\frac{1}{2}\int w^2 dP-\int w dQ. 
\end{align*}
Both $L_{\rm KMM}$ and $L_{\rm KuLSIF}$ are
minimized at the true density ratio $w_0\in{\cal H}$. 
Although some linear constraints may be introduced in the optimization phase, we study 
the optimization problems of $L_{\rm KMM}$ and $L_{\rm KuLSIF}$ without constraints. 
This is because when the sample size tends to infinity, the optimal solutions of $L_{\rm KMM}$
and $L_{\rm KuLSIF}$ without constraints automatically satisfy the required constraints
such as $\int w dP=1$ and $w\geq 0$. 

We consider the extremal condition of $L_{\rm KuLSIF}(w)$ at $w_0$. Substituting
$w=w_0+\delta\cdot v\,(\delta\in\Real,\, v\in\HC)$ into $L_{\rm KuLSIF}(w)$,
we have 
\begin{align*}
L_{\rm KuLSIF}(w_0+\delta v)-L_{\rm KuLSIF}(w_0)
~=~
 \delta\left\{\int w_0 v dP-\int vdQ\right\}
 +\frac{\delta^2}{2}\int v^2dP. 
\end{align*}
Since $L_{\rm KuLSIF}(w_0+\delta v)$ is minimized at $\delta=0$, 
the derivative of $L_{\rm KuLSIF}(w_0+\delta v)$ at $\delta=0$ vanishes, i.e.,
\begin{align}
\int w_0 v dP-\int vdQ=0. 
 \label{eqn:extremal-cond-kernel-uLSIF}
\end{align}
The equality \eqref{eqn:extremal-cond-kernel-uLSIF} holds for arbitrary $v\in{\cal H}$. 
Using the reproducing property of the kernel function $k$, we can express
\eqref{eqn:extremal-cond-kernel-uLSIF} in terms of $\Phi(w_0)$ as follows, 
\begin{align}
\int w_0 v dP-\int vdQ
&~=~ \int w_0(x) \<k(\cdot,x),v\>_{\mathcal{H}} P(dx)
 -\int \<k(\cdot,y),v\>_{\mathcal{H}} Q(dy)\nonumber\\ 
&~=~
 \big\< \int k(\cdot,x) w_0(x)  P(dx)
 -\int k(\cdot,y) Q(dy),\ v \big\>_{\mathcal{H}}\nonumber\\
&~=~
 \big\< \Phi(w_0),\ v \big\>_{\mathcal{H}}=0,\quad ^\forall v\in{\cal H}. 
 \label{eqn:extremal-cond-written-Phi}
\end{align}
Therefore, we obtain $\Phi(w_0)~=~0$
and we find that $\Phi(w)$ is the G\^{a}teaux derivative
\cite{zeidler86:_nonlin_funct_analy_and_its_applic_i} of $L_{\rm KuLSIF}$ at 
$w\in{\cal H}$. In summary, let $DL_{\rm KuLSIF}$ be the G\^{a}teaux derivative of 
$L_{\rm KuLSIF}$ over the RKHS $\mathcal{H}$, then, the equality
\begin{align}
 \label{eqn:relatoin-KMM-KuLSIF-derivative}
 L_{\rm KMM}(w)~=~\frac{1}{2}\|DL_{\rm KuLSIF}(w)\|_{\mathcal{H}}^2 
\end{align}
holds. \emcite{tsuboi08:_direc_densit_ratio_estim_for} have pointed out a similar relation
for M-estimator based on Kullback-Leibler divergence. 

Now we illustrate the relation between KuLSIF and KMM by showing
an analogous optimization example in the Euclidean space.
Let $f:\Real^d\rightarrow\Real$ be a
differentiable function, and consider the optimization problem 
$\min_{x} f(x)$. At the optimal solution $x_0$, the extremal condition $\nabla f(x_0)=0$ 
should hold, where $\nabla f$ is the gradient of $f$. Thus, instead of minimizing $f$,
minimization of $\|\nabla f(x)\|^2$  also provides the minimizer of $f$. 
This corresponds to the relation between KuLSIF and KMM:
\begin{align*}
        \text{KuLSIF} &\  \Longleftrightarrow\   \min_x f(x),\\
 \text{KMM}&\  \Longleftrightarrow\  \min_x\ \frac{1}{2}\|\nabla f(x)\|^2. 
\end{align*}

In other words, in order to find the solution of the equation
\begin{align}
 \Phi(w) = 0, \label{eqn:equation-to-be-solved}
\end{align}
KMM tries to minimize the norm of $\Phi(w)$. 
The ``dual'' expression of \eqref{eqn:equation-to-be-solved} is given as
\begin{align}
 \<\Phi(w),v\>_{\mathcal{H}}=0, \ \ ^\forall v\in{\cal H}. 
 \label{eqn:dual-equation-to-be-solved}
\end{align}
By ``integrating'' $\<\Phi(w),v\>_{\mathcal{H}}$, we obtain the loss function 
$L_{\rm KuLSIF}$. 

\begin{remark}
\emcite{gretton06:_kernel_method_for_two_sampl_probl} have proposed the maximum mean 
discrepancy (MMD) to measure the discrepancy between two probabilities $P$ and $Q$. 
When the constant function $1$ is included in the RKHS $\mathcal{H}$, 
the MMD between $P$ and $Q$ is equal to $2\times L_{\rm KMM}(1)$. 
Due to the equality \eqref{eqn:relatoin-KMM-KuLSIF-derivative}, we find that the MMD is
also expressed as $\|DL_{\rm KuLSIF}(1)\|_{\mathcal{H}}^2$, that is, the norm of the
derivative of $L_{\rm KuLSIF}$ at $1\in\mathcal{H}$. 
This quantity will be related to the discrepancy between the constant function $1$ and the
true density ratio $w_0=q/p$. 
\end{remark}

\begin{remark}
\label{remark:f-KMM}
It is straightforward to extend the above relation to the general $f$-divergence approach.
The loss function of the M-estimator \cite{NIPS:Nguyen+etal:2008} is given as
\begin{align*}
 L_\psi(w)~=~\int \psi(w)dP-\int w dQ. 
\end{align*}
Then, the loss function of the KMM-type may be defined as 
\begin{align*}
 L_{\text{$\psi$-KMM}}(w)~=~\frac{1}{2}\|DL_\psi(w)\|_{\cal H}^2, 
\end{align*}
where 
\begin{align*}
 DL_{\text{$\psi$-KMM}}(w)~=~
 \int k(\cdot,x) \psi'(w(x)) P(dx)-\int k(\cdot,y)Q(dy). 
\end{align*}
We can confirm that $L_\psi(w)$ and $L_{\text{$\psi$-KMM}}(w)$ share the minimizer. If
there exists $w_\psi\in{\cal H}$ such that $w_0=\psi'(w_\psi)$, the optimal solution is
given by $w_\psi$. 
\end{remark}

\section{Condition Number Analysis for Density Ratio Estimation}
\label{sec:condition-number}
We have elucidated basic properties of the KuLSIF algorithm. In this section, we study the
condition number of KuLSIF and other density ratio estimators in order to investigate
computational properties. This is the main contribution of this paper. 


\subsection{Condition Number in Numerical Analysis and Optimization}
\label{sec:Condition_Number_in_Optimization}
Condition numbers play crucial roles in numerical analysis and optimization 
\cite{demmel97:_applied_numer_linear_algeb,luenberger08:_linear_and_nonlin_progr,sankar06:_smoot_analy_of_condit_number},
which is explained in this section. 

Let $A$ be a symmetric positive definite matrix, and the condition number of $A$ 
is defined as $\lambda_{\max}/\lambda_{\min}\,(\geq 1)$, 
where $\lambda_{\max}$ and $\lambda_{\min}$ are the maximal and minimal eigenvalues of $A$,
respectively. The condition number of $A$ is denoted by $\kappa(A)$. 
In general, the condition number for a matrix which may not be symmetric is defined
through the singular values. The above definition is, however, enough for our purpose. 

In numerical analysis, the condition number governs the round-off error of the solution of
a linear equation $Ax=b$. The matrix $A$ with a large condition number will lead to a
large upper bound on the relative error of the solution $x$. 
More precisely, in the perturbed linear equation $(A+\delta A)(x+\delta x)=b+\delta b$,
the relative error of the solution is given as follows 
\cite{demmel97:_applied_numer_linear_algeb}: 
\begin{align*}
 \frac{\|\delta x\|}{\|x\|}~\leq~
 \frac{\kappa(A)}{1-\kappa(A)\|\delta A\|/\|A\|}
 \left(
 \frac{\|\delta A\|}{\|A\|}+\frac{\|\delta b\|}{\|b\|}
 \right). 
\end{align*}
Hence, smaller condition number is preferable in numerical computation. 

In optimization problems, the condition number determines the convergence rate of
optimization algorithms. 
Let us consider a minimization problem $\min_x f(x),\  x\in\Real^n$, where
$f:\Real^n\rightarrow\Real$ is a differentiable function
and let $x_0$ be a local optimal solution. 
We consider an iterative algorithm which generates a sequence
$\{x_i\}_{i=1}^\infty$. 
In various iterative algorithms, the sequence is generated as
\begin{equation}
  x_{i+1}=x_i-S_i^{-1}\nabla f(x_i),\quad i=1,2,\ldots, 
 \label{eqn:iterative_method}
\end{equation}
where $S_i$ is an approximation of the Hessian matrix of $f$ at $x_0$, i.e.,
$\nabla^2 f(x_0)$. 
Then under a mild assumption, the sequence $\{x_i\}_{i=1}^\infty$ converges to $x_0$. 
Numerical techniques such as scaling and pre-conditioning are also
incorporated in the above form with a certain choice of $S_i$.
According to Section 10.1 in \emcite{luenberger08:_linear_and_nonlin_progr}, the
convergence rate of such iterative algorithms is given as 
\[
\|x_k-x_0\|=O\bigg(\prod_{i=1}^k \frac{\kappa_i-1}{\kappa_i+1}\bigg), 
\]
where $\kappa_i$ is the 
condition number of $S_i^{-1/2}(\nabla^2 f(x_0)) S_i^{-1/2}$. 
Thus, the convergence rate of the sequence $x_k$ is slow if $\kappa_i$ is large. 
More critically, when $\{\kappa_i\}_{i=1}^\infty$ does not converge to one, the sequence
$\{x_i\}_{i=1}^\infty$ does not converge to $x_0$ at a super-linear rate. 

When the condition number of the Hessian matrix $\nabla^2 f(x_0)$ is large, there is a
trade-off between the numerical accuracy and the convergence rate in optimization 
problems. 
Let us illustrate the trade-off using a few examples. 
When the Newton method is employed, $S_k$ is given as $\nabla^2 f(x_k)$. 
Because of the continuity of $\nabla^2 f$, 
the condition number of $S_k=\nabla^2 f(x_k)$ would be large if $\kappa(\nabla^2 f(x_0))$
is large. 
Then the numerical computation of $S_k^{-1}\nabla f(x_k)$ becomes unstable. 
When the quasi-Newton methods such as the BFGS method or the DFP method
\cite{luenberger08:_linear_and_nonlin_progr} are employed, $S_k$ or $S_k^{-1}$ is 
successively estimated based on the information of the gradient. 
If $\kappa(\nabla^2 f(x_0))$ is large, $\kappa(S_k)$ is also likely to be large, and thus, 
the numerical computation of $S_k^{-1}\nabla f(x_k)$ is not reliable, even when $S_k^{-1}$
is successively updated in the quasi-Newton methods. 
The round-off error caused by nearly singular Hessian matrices 
significantly affects the accuracy of the quasi-Newton methods.
As a result, it may not be guaranteed that $S_k^{-1}\nabla f(x_k)$ is a preferable descent
direction of the objective function $f$. 

In optimization problems with large condition numbers, the numerical 
computation tends to be unreliable. To avoid numerical instability, 
the Hessian matrix is often modified so that $S_k$ has a moderate condition number. 
For example, the optimization toolbox in MATLAB$^\text{\textregistered}$
implements a gradient descent method in its function \texttt{fminunc}.
The default method in \texttt{fminunc} is the BFGS method
with update through the Cholesky factorization of $S_k$ (not $S_k^{-1}$).
Even if the positive definiteness of $S_k$ is violated by the round-off error,
the Cholesky factorization immediately detects the negativity of eigenvalues
and the positive definiteness of $S_k$ is
recovered by adding a correction term. When the modified Cholesky factorization 
is used, the condition number of $S_k$ is guaranteed to be bounded above by some constant,
$C$. See \cite{more84:_newton_method} in details. 

The trade-off between numerical accuracy and convergence rate is summarized by the
following equality: 
\begin{align}
 \min_{S:\kappa(S)\leq C} \kappa(S^{-1/2} (\nabla^2 f(x_0)) S^{-1/2}) 
 ~=~
 \max\big\{\,\frac{\kappa(\nabla^2 f(x_0))}{C},\,1\,\big\}. 
 \label{eqn:trade-off_accuracy_speed}
\end{align}
The proof of \eqref{eqn:trade-off_accuracy_speed} may be found in Appendix 
\ref{appendix:condition-number-equality}. 
We suppose that the symmetric positive definite matrix $S_k$ satisfying 
$\kappa(S_k)\leq C$ is used in the iterative algorithm \eqref{eqn:iterative_method}. 
If $\kappa(\nabla^2 f(x_0))$ is large, 
the right-hand side of \eqref{eqn:trade-off_accuracy_speed} will be greater than
one. Hence, the convergence rate will be slow. 
That is, the quasi-Newton method with a modified Hessian $S_k$ such that 
$\kappa(S_k)\leq C$ may not achieve a super-linear convergence rate. 
Even though some scaling or pre-conditioning technique is available,
it is preferable that the condition number of the original problem is 
kept as small as possible.

\subsection{Condition Number Analysis of KuLSIF and KMM}

Let us consider the optimization problems in KuLSIF and KMM on an RKHS $\mathcal{H}$
endowed with a kernel function $k$ over a set $\mathcal{Z}$. 
Given samples \eqref{eqn:samples}, the optimization problems of KuLSIF and KMM are 
defined as 
\begin{align*}
 \text{(KuLSIF)}&\ \ \min_w\ \frac{1}{2}\int\! w^2 dP_n\!-\!\!\int\! w dQ_m 
 + \frac{\lambda}{2}\|w\|_\mathcal{H}^2,\quad  w\in\mathcal{H},\\
 \text{(KMM)}&\ \ \min_w\ \frac{1}{2}
 \big\|\widehat{\Phi}(w) +\lambda w\big\|_\mathcal{H}^2,\quad  w\in \mathcal{H}, 
\end{align*}
where
\begin{align*}
\widehat{\Phi}(w)=\int k(\cdot,x)w(x)P_n(dx)-\int k(\cdot,y)Q_m(dy). 
\end{align*}
Here, $\widehat{\Phi}(w) +\lambda w$ is the G\^{a}teaux derivative of the loss function
for KuLSIF including the regularization term. 
In the original KMM method, the density ratio on samples $X_1,\ldots,X_n$
are optimized \cite{NIPS2006_915}, i.e., transduction. 
Here, we consider its inductive variant, i.e., estimating the function
$w_0$ on $\mathcal{Z}$ using the loss function of KMM. 
According to Theorem~\ref{thm:beta-const}, the optimal solution of (KuLSIF) is given as
the form of
$w=\sum_{i=1}^{n}\alpha_ik(\cdot,X_i)+\frac{1}{m\lambda}\sum_{j=1}^{m}k(\cdot,Y_j)$; 
note that the optimal solution of (KMM) is also given by the same form.
Thus, the variables to be optimized in (KuLSIF) and (KMM) are $\alpha_1,\ldots,\alpha_n$. 

We investigate the numerical efficiency of (KuLSIF) and (KMM). 
When we solve the minimization problem $\min_x f(x)$, it is not recommended to minimize 
the norm of the gradient $\min_x \|\nabla f(x)\|^2$, 
since the problem $\min_x \|\nabla f(x)\|^2$ generally has a larger condition number than
$\min_x f(x)$ \cite{luenberger08:_linear_and_nonlin_progr}. 
For example, let $f$ be the convex quadratic function defined as 
$f(x)=\frac{1}{2}x^\top A x -b^\top x$ with a positive-definite matrix $A$. 
Then the condition number of the Hessian matrix equals to $\kappa(A)$. 
On the other hand, the Hessian matrix of the function $\|\nabla f(x)\|^2=\|Ax-b\|^2$ is
equal to $\kappa(A^2)=\kappa(A)^2$, that is, the condition number is squared and thus
becomes larger. Below, we show that the same is true of KuLSIF and KMM. 

The Hessian matrices of the objective functions of KuLSIF and KMM are given as 
\begin{align}
 H_\mathrm{KuLSIF}&~=~\frac{1}{n}K_{11}^2+\lambda K_{11}, 
 \label{eqn:Hessian-KuLSIF}\\
 H_\mathrm{KMM}&~=~\frac{1}{n^2}K_{11}^3+\frac{2\lambda}{n}K_{11}^2+\lambda^2K_{11}. 
 \label{eqn:Hessian-KMM} 
\end{align}
$H_\mathrm{KuLSIF}$ is derived from \eqref{eqn:KuLSIF-opt-prob}, and 
$H_\mathrm{KMM}$ is given by direct computation based on (KMM). 
Then, we obtain 
\begin{align*}
 \kappa(H_\mathrm{KuLSIF}) ~=~&
 \kappa(K_{11})\kappa\big(\frac{1}{n}K_{11}+\lambda I_n\big),\\
 \kappa(H_\mathrm{KMM}) ~=~&\kappa(K_{11})\kappa\big(\frac{1}{n}K_{11}
 +\lambda I_n\big)^2. 
\end{align*}
Since the condition number is larger than or equal to one, 
the inequality 
\[
\kappa(H_\mathrm{KuLSIF})~\leq~\kappa(H_\mathrm{KMM})
\]
holds. This implies that the convergence rate of KuLSIF well be faster than that of KMM,
when an iterative optimization algorithm is used to minimize each loss function. 

According to Remark~\ref{remark:f-KMM}, we expect that 
the condition number of M-estimator based on $L_\psi$ is smaller than that of KMM based on
$L_{\psi\text{-KMM}}$. 
Let each Hessian matrix at optimal solution $\what$ be $H_{\psi\text{-div}}$ for $L_\psi$
and $H_{{\psi\text{-KMM}}}$ for $L_{\psi\text{-KMM}}$, then some calculation provides 
\begin{align*}
 H_{\psi\text{-div}} =~&
 K_{11}^{1/2}\left(\frac{1}{n}K_{11}^{1/2}D_{\psi,\what}K_{11}^{1/2}
 +\lambda I_n\right)K_{11}^{1/2},\\
 H_{{\psi\text{-KMM}}} ~=~&
 K_{11}^{1/2}\left(\frac{1}{n}K_{11}^{1/2}D_{\psi,\what}K_{11}^{1/2}
 +\lambda I_n\right)^2K_{11}^{1/2}, 
\end{align*}
where $D_{\psi,w}$ is the $n$ by $n$ diagonal matrix defined as 
\begin{align}
 D_{\psi,w} ~=~
 \begin{pmatrix}
  \psi''(w(X_1))& & \\	&\ddots & \\	& & \psi''(w(X_n))
 \end{pmatrix},
 \label{eqn:def-of-D_psi_w}
\end{align}
and $\psi''$ denotes the second-order derivative of $\psi$.
Hence, using the inequality $\kappa(AB)\leq \kappa(A)\kappa(B)$
\cite{horn85:_matrix_analy}, we have 
\begin{align*}
\kappa(H_{\psi\text{-div}}) 
 ~\leq~& 
 \kappa(K_{11})  \kappa\big(\frac{1}{n}K_{11}^{1/2}D_{\psi,\what}K_{11}^{1/2} 
 +\lambda I_n\big),\\
 \kappa(H_{{\psi\text{-KMM}}})
 ~\leq~& 
 \kappa(K_{11})  \kappa\big(\frac{1}{n}K_{11}^{1/2}D_{\psi,\what}K_{11}^{1/2} 
 +\lambda I_n\big)^2. 
\end{align*}
From the viewpoint of the naive upper bound of condition numbers,
the M-estimator based on
$L_\psi$ will be preferable to KMM with $L_{\psi\text{-KMM}}$.

\subsection{Condition Number Analysis of $M$-Estimators}\label{sec:Optimization_in_KuLSIF_and_KMM}
(K)uLSIF is an example of the M-estimators with the squared loss. 
Here, we study the condition number of the Hessian matrix 
associated with the minimization problem in the $f$-divergence approach, and show that
KuLSIF is optimal among all M-estimators based on $f$-divergences.  
More specifically, we will give a min-max evaluation (Section~\ref{sec:Deterministic_Evaluation})
and a probabilistic evaluation (Section~\ref{sec:Probabilistic_Evaluation}) 
of the condition number.




\subsubsection{Min-max Evaluation}\label{sec:Deterministic_Evaluation}

We assume that a universal RKHS ${\cal H}$ 
\cite{JMLR:Steinwart:2001}
endowed with a kernel function $k$ on a compact set $\mathcal{Z}$ is used for estimation
of $w_0$. 
The M-estimator based on the $f$-divergence is obtained by solving the problem
\eqref{eqn:f-div-RKHS-est-representer}. 
The Hessian matrix of the loss function at the optimal solution $w$ is equal to 
\begin{align}
 \frac{1}{n}K_{11}D_{\psi,w} K_{11}+\lambda K_{11}, 
 \label{eqn:Hessian_matrix_f-div}
\end{align}
where $D_{\psi,w}$ is the diagonal matrix defined as Eq.\,\eqref{eqn:def-of-D_psi_w}. 
The condition number of the Hessian matrix is denoted by
\begin{align*}
 \kappa_0(D_{\psi,w})~=~\kappa\bigg(\frac{1}{n}K_{11}D_{\psi,w}K_{11}+\lambda K_{11}\bigg). 
\end{align*}
In KuLSIF, we find $\psi''=1$, and thus, the condition number is equal to
$\kappa_0(I_n)$. 
We analyze the relation
between $\kappa_0(I_n)$ and $\kappa_0(D_{\psi,w})$. 

\begin{theorem}[Min-max Evaluation]
 \label{thm:min-max-evaluation}
 Suppose that $\mathcal{H}$ is a universal RKHS, and that $K_{11}$ is non-singular. 
 Then, 
\begin{align}
 \inf_{\psi:\psi''(1)=1}\,\sup_{w\in{\cal H}}
 \kappa_0(D_{\psi,w})~=~\kappa_0(I_n)
 \label{eqn:min-max-evaluation}
\end{align}
 holds. Here the infimum is taken over all convex second-order continuously differentiable
 functions $\psi$ such that $\psi''(1)=1$. 
\end{theorem}
The proof is deferred to Appendix \ref{appendix:Proof_of_min-max-theorem}. 
When the constraint $\psi''(1)=c$ is imposed with some $c>0$, the optimal function is
given as $\psi(z)=cz^2/2$ in the min-max sense. Practically, the value of $\psi''(1)$
determines the balance between the fitting to training samples and the regularization term. 
Theorem~\ref{thm:min-max-evaluation} guarantees that KuLSIF minimizes the worst-case
condition number, which is brought by the fact that 
the condition number of KuLSIF does not depend on the optimal solution. 
Since both sides of \eqref{eqn:min-max-evaluation} depend on the samples $X_1,\ldots,X_n$, 
KuLSIF achieves the min-max solution in
terms of the condition number for \emph{each} observation. 

\subsubsection{Probabilistic Evaluation}\label{sec:Probabilistic_Evaluation}
Next, we study 
probabilistic evaluation of the condition number. 
 As shown in min-max evaluation, the Hessian matrix is given as 
 \begin{align*}
  H=\frac{1}{n}K_{11}D_{\psi,\what} K_{11} + \lambda K_{11}, 
 \end{align*}
where the diagonal elements of $D_{\psi,\what}$ are equal to 
$\psi''(\what(X_1)),\ldots,\psi''(\what(X_n))$. 
The estimator $\what$ is given as the minimum solution of
\eqref{eqn:f-div-RKHS-est-representer}. 
Let us define the random variable $T_n$ as 
\begin{align*}
 T_n ~=~ \max_{1\leq i\leq n}\, \psi''(\what(X_i)), 
\end{align*}
and $F_n$ be the distribution function of $T_n$, then $T_n$ is a non-negative random
variable. 

Below, we first compute the distribution of the condition number $\kappa(H)$. Then we
investigate the relation between the function $\psi$ and the distribution of condition
number $\kappa(H)$. 
We need to study the eigenvalues and the condition numbers of random matrices. 
For the Wishart distribution, the probability distribution of condition numbers has been 
investigated by 
\emcite{edelman88:_eigen_and_condit_number_of_random_matric,edelman05:_tails_of_condit_number_distr}. 
Recently, the condition number of matrices perturbed by additive Gaussian noise have been 
investigated by the name of \emph{smoothed analysis} 
\cite{sankar06:_smoot_analy_of_condit_number,spielman04:_smoot_analy_of_algor,vu07:_condit_number_of_random_pertur_matrix}. 
Randomness involved in the matrix $H$ defined above is, however, different from that in
existing works. 

\begin{theorem}[Probabilistic Evaluation]
\label{cor:cond-num-prob-convergence} 
 Let $\mathcal{H}$ be a RKHS endowed with a kernel function $k$ on $\mathcal{Z}$ 
 satisfying the following condition: there exists $\varepsilon>0$ such that 
 \begin{align*}
  \sqrt{\varepsilon} \leq k(x,x') \leq 1,\quad \forall x,x'\in\mathcal{Z}. 
 \end{align*}
 Assume that the Gram matrix $K_{11}$ is almost surely positive definite in terms of the
 probability measure $P$. 
 Suppose that there exists sequences $s_n$ and $t_n$ such that
 \begin{align}
  \lim_{n\rightarrow\infty} s_n = \infty, \quad
  \lim_{n\rightarrow\infty} F_n(s_n)=0, \quad
  \lim_{n\rightarrow\infty} F_n(t_n) = 1,
  \label{eqn:sequences_for_convergence}
 \end{align}
 and that 
 there exists $M>0$ such that $E[\psi''(\what(X_1))]\leq M$ holds for large sample size,
 $n$ and $m$. 
 Suppose that $\lambda=\lambda_{n,m}$ satisfies
 $\lim_{n\rightarrow\infty}\lambda_{n,m}<\infty$. 
 Then, for any small $\nu>0$, we have
 \begin{align}
 \lim_{n\rightarrow\infty}
 \Pr\left(
  s_n^{1-\nu}
  \leq \kappa(H) \leq 
  \kappa(K_{11})\big(1+\frac{t_n}{\lambda}\big)
 \right)=1. 
  \label{eqn:probabilisitic-bound}
 \end{align}
\end{theorem}

The proof is deferred to Appendix \ref{appendix:proof-of-probability-eval-theorem}. 

\begin{remark}
 The Gaussian kernel on a compact set meets the condition of Theorem 
 \ref{cor:cond-num-prob-convergence} under a mild assumption on the probability $P$. 
 If the distribution $P$ of samples $X_1,\ldots,X_n$ is absolutely continuous with respect
 to the Lebesgue measure, the Gram matrix of the Gaussian kernel is almost surely positive
 definite. Because, $K_{11}$ is positive definite if $X_i\neq X_j$ for $i\neq j$. 
\end{remark}

When $\psi$ is the quadratic function, $\psi(z)=z^2/2$, the distribution function $F_n$ is
given $F_n(t)=\1[t\geq 1]$, where $\1[\,\cdot\,]$ is the indicator function. Hence, there
does not exist a sequence $s_n$ defined in Theorem
\ref{cor:cond-num-prob-convergence}. The upper bound is, however, still valid. That is, by 
choosing $t_n=1$, the upper bound of $\kappa(H)$ with $\psi(z)=z^2/2$ is asymptotically
given as $\kappa(K_{11})(1+\lambda_{n,m}^{-1})$. On the other hand, in the M-estimator
with Kullback-Leibler divergence \cite{NIPS:Nguyen+etal:2008}, 
the function $\psi$ is defined as $\psi(z)=-1-\log(-z),\ z<0$, and thus, $\psi''(z)=1/z^2$
holds. 
Hence, $T_n=\max_{1\leq i\leq n}(\widehat{w}(X_i))^{-2}$ is expected to be of the order
larger than constant order, and thus, $t_n$ would diverge to infinity. This simple
analysis indicates that the KuLSIF will be more preferable than the M-estimator with
Kullback-Leibler divergence in the sense of computational efficiency and stability. 

We derive an approximation of the inequality in \eqref{eqn:probabilisitic-bound}. 
The target of the estimator $\what$ is given as $w$ such that 
$q(x)/p(x) ~=~ \psi'(w(x))$ holds. 
Thus, we expect that the condition number of 
$\frac{1}{n}K_{11}D_{\psi,\widehat{w}}K_{11}+\lambda K_{11}$
is approximated by that of $\frac{1}{n}K_{11}D_{\psi,w}K_{11}+\lambda K_{11}$. 
The proof of Theorem \ref{cor:cond-num-prob-convergence} is valid even in the case that
the random variable $T_n$ is defined by a fixed function $w\in\mathcal{H}$. 
The condition number of Hessian matrix at a fixed function $w\in\mathcal{H}$ is considered 
in the proposition below. 
\begin{proposition}[Approximated Bound]
 \label{proposition:Approximated-Bound}
 The kernel function $k$ and the regularization parameter $\lambda$ satisfy the same
 condition as Theorem \ref{cor:cond-num-prob-convergence}. 
 For a function $w\in\mathcal{H}$, let $F$ be the distribution function of $\psi''(w(X))$, 
 and suppose that the expectation of $\psi''(w(X))$ is finite. 
 Let $G$ be $1-F$, and suppose that there exists a real number $U>0$ such that $G(t)$ has
 the inverse function $G^{-1}$ for $t\geq U$. 
 Let the random matrix $H_w$ be
 \[
  H_w=\frac{1}{n}K_{11}D_{\psi,w}K_{11}+\lambda K_{11}. 
 \]
 Then, for any small $\eta>0$ and any small $\nu>0$, we have
 \[
 \lim_{n\rightarrow\infty}
 \Pr\left(
 \{G^{-1}(1/n^{1-\eta})\}^{1-\nu}
 \leq \kappa(H_w) \leq 
 \kappa(K_{11})\left(1+\lambda^{-1}G^{-1}(1/n^{1+\eta})\right)
 \right)=1. 
 \]
\end{proposition}
\begin{proof}
 Note that $F_n(t)$ in Theorem \ref{cor:cond-num-prob-convergence} is equal to $(F(t))^n$, 
 since $\psi''(w(X_i)),\,i=1,\ldots,n$ are identically and independently distributed form
 $F$. The condition number $\kappa(H_w)$ satisfies Eq.\eqref{eqn:probabilisitic-bound}
 with $F_n=F^n$. 

 As shown in Figure~\ref{fig:cond-shift}, the function $G^{-1}$ is decreasing. 
 Let $s_n$ be $s_n=G^{-1}(1/n^{1-\eta})$, then
 $s_n\rightarrow\infty$ holds when $n$ tends to infinity. 
 Thus, we have 
 \[
 F_n(s_n)=F(s_n)^n=(1-G(s_n))^n=
 \left(1-\frac{1}{n^{1-\eta}}\right)^n\ \longrightarrow\ 0,\ \ n\rightarrow\infty.
 \]
 On the other hand, 
 let $t_n$ be $t_n=G^{-1}(1/n^{1+\eta})$, then
 we have
 \[
  F_n(t_n)=(1-G(t_n))^n
 =\left(1-\frac{1}{n^{1+\eta}}\right)^n\ \longrightarrow\ 1,\ \ n\rightarrow\infty.
 \]
 Substituting $s_n$ and $t_n$ into the inequality in \eqref{eqn:probabilisitic-bound}, we
 obtain the result. 
\end{proof}

\begin{remark}
 \label{remark:KuLSIF-smallest-cond-num}
 Proposition \ref{proposition:Approximated-Bound} implies that for large $n$, the
 inequality 
 \begin{align}
  \{G^{-1}(1/n^{1-\eta})\}^{1-\nu} ~\leq~ \kappa(H_w) ~\leq~
 \kappa(K_{11})\left(1+\lambda^{-1}G^{-1}(1/n^{1+\eta})\right)
  \label{remark:prob-bound}
 \end{align}
 holds in high probability. 
 In KuLSIF, the function 
 $\psi$ is given as $\psi(z)=z^2/2$, and the
 corresponding distribution function of each diagonal element in $D_{\psi,w}$ 
 is given by $F_{\mathrm{KuLSIF}}(d)=\1[d\geq 1]$, and thus, 
 $G_{\mathrm{KuLSIF}}(d)=1-F_{\mathrm{KuLSIF}}(d)=\1[d<1]$. 
 In all M-estimators except KuLSIF, diagonal elements
 of $D_{\psi,w}$ can take various positive values. 
 We regard the diagonal elements of $D_{\psi,w}$ 
 as a typical realization of random variables with the distribution function $F(d)$. 
 When the distribution function $F$ is close to $F_{\mathrm{KuLSIF}}$, 
 the function $G=1-F$ is also close to $G_{\mathrm{KuLSIF}}$. 
 Then, $G^{-1}$ will take small values as illustrated in Figure~\ref{fig:cond-shift}. 
 As a result, we can expect that the condition number of KuLSIF is smaller than that of
 the other M-estimators. In a later section, we further investigate this issue through
 numerical experiments. 
\end{remark}

\begin{figure}
 \begin{center}
 \scalebox{0.5}{\includegraphics{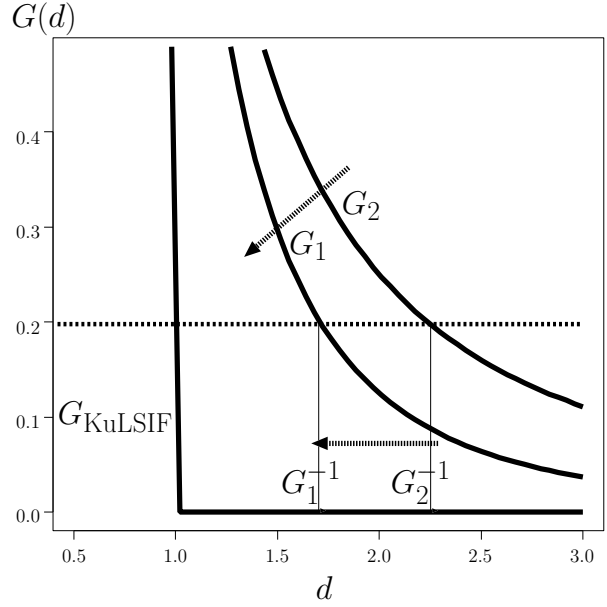}}\vspace*{-5mm}
 \end{center}
 \caption{If the function $G_1(d)$ is closer to $G_{\mathrm{KuLSIF}}(d)\;(=0)$ 
 than $G_2(d)$ for large $d$, then 
 $G_1^{-1}(z)$ takes smaller value than $G_2^{-1}(z)$ for small $z$. }
 \label{fig:cond-shift}
\end{figure}

\begin{example}
 \label{example:power-decay}
 Let $F_\gamma(d)$ be
 \begin{align*}
  F_\gamma(d) ~=~
  \begin{cases}
   0& 0\leq d <1,\\
   1-\frac{1}{d^\gamma}& 1\leq d.
  \end{cases}
 \end{align*}
 Suppose that $F_\gamma$ is the distribution function of
 $\psi''(w(X))=\psi''(\psi'^{-1}(q(X)/p(X)))$. 
 Note that the distribution function $F_{\mathrm{KuLSIF}}(d)=\1[d\geq 1]$ is represented as 
 $\1[d\geq 1]=\lim_{\gamma\rightarrow\infty}F_\gamma(d)$ except at $d=1$. 
 Then, $G_\gamma(d)=1-F_\gamma(d)$ is equal to
  \begin{align*}
  G_\gamma(d) ~=~
   \begin{cases}
    1& 0\leq d <1,\\
    \frac{1}{d^\gamma}& 1\leq d.
   \end{cases} 
 \end{align*}
 For small $z>0$, the inverse function $G_\gamma^{-1}(z)$ is given as
 \begin{align*}
  G_\gamma^{-1}(z)~=~z^{-1/\gamma}. 
 \end{align*}
 Hence for sufficiently small $\eta$, the inequality \eqref{remark:prob-bound} is reduced
 to 
 \[
 n^{\frac{(1-\eta)(1-\nu)}{\gamma}}~\leq~\kappa(H_w)~\leq~
 \kappa(K_{11})\big(1+\lambda^{-1}n^{\frac{1+\eta}{\gamma}}\big). 
 \]
 Both upper and lower bounds in the above inequality are monotone decreasing with respect
 to $\gamma$. 
\end{example}

\begin{example}
 \label{example:exponential-decay}
 Let $F_\gamma(d)$ be
 \begin{align*}
 F_\gamma(d) ~=~ \frac{1}{1+e^{-\gamma(d-1)}},\quad d\geq 0.
 \end{align*}
 The distribution function $F_{\mathrm{KuLSIF}}(d)=\1[d\geq 1]$ is represented as
 $\1[d\geq 1]=\lim_{\gamma\rightarrow\infty}F_\gamma(d)$ except at $d=1$.
 Then, $G_\gamma(d)=1-F_\gamma(d)$ is equal to
 \begin{align*}
 G_\gamma(d) ~=~ \frac{1}{1+e^{\gamma(d-1)}},\quad d\geq 0.
 \end{align*}
 For small $z$, the inverse function $G_\gamma^{-1}(z)$ is given as
 \begin{align*}
 G_\gamma^{-1}(z)~=~1+\frac{1}{\gamma}\log\frac{1-z}{z}.
 \end{align*}
 Hence for small $\eta$, the inequality \eqref{remark:prob-bound} will lead the following:
 \[
 \bigg(\frac{1-\eta}{\gamma}\log\frac{n}{2}\bigg)^{1-\nu}
 ~\leq~\kappa(H_w)~\leq~
 \kappa(K_{11})\cdot \frac{1+\eta}{\lambda \gamma}\log n. 
 \]
 The upper and lower bounds in the above inequality are monotone decreasing with respect
 to $\gamma$. 
\end{example}

\section{Reduction of Condition Numbers in KuLSIF} \label{sec:reduction_cond_num}
The condition number in the optimization problem of KuLSIF is given as
$\kappa(H_{\mathrm{KuLSIF}})=\kappa(\frac{1}{n}K_{11}^2+\lambda K_{11})$, and  
that of the original KMM method is equal to $\kappa(K_{11})$ which is approximately
derived from \eqref{eqn:kmm-opt-prob}. 
On the other hand, the Hessian matrix of R-KuLSIF is equal to 
\begin{align}
\label{eqn:Hessian_RKuLSIF}
 H_{\mathrm{R-KuLSIF}}~=~\frac{1}{n}K_{11}+\lambda I_n. 
\end{align}
See \eqref{eqn:reduced-KuLSIF} for the loss function of R-KuLSIF. 
Due to the equality
\begin{align*}
 \kappa(H_{\mathrm{KuLSIF}}) ~=~ \kappa(K_{11}) \kappa(H_{\mathrm{R-KuLSIF}}),
\end{align*}
we have 
\begin{align*}
\kappa(H_{\mathrm{R-KuLSIF}}) ~\leq~ \kappa(H_{\mathrm{KuLSIF}}). 
\end{align*}
Moreover, it is easy to see 
\begin{align*}
\kappa(H_{\mathrm{R-KuLSIF}}) ~\leq~ \kappa(K_{11}). 
\end{align*}
These inequalities imply that R-KuLSIF is more preferable than KuLSIF and KMM in the sense 
of the convergent speed and numerical stability as explained in Section
\ref{sec:Condition_Number_in_Optimization}. 

In this section, we study 
whether reduction of condition numbers is possible in the general 
$f$-divergence approach. We do not consider scaling of the parameter
\cite{luenberger08:_linear_and_nonlin_progr},
but other types of transformation of loss functions 
in order to reduce the condition number. 
Our conclusion is that among all $f$-divergence approaches,
the condition number is reducible only in KuLSIF.
Thus the reduction of condition numbers by R-KuLSIF is a special property,
which makes R-KuLSIF particularly attractive in practical use.

We elucidate the reason why the condition number of KuLSIF can be reduced from 
$\kappa(H_{\mathrm{KuLSIF}})$ to $\kappa(H_{\mathrm{R-KuLSIF}})$. 
As explained in Remark~\ref{rem:beta-const}, 
in the $f$-divergence approach, the optimal solution of $\beta$ is equal to
$\1_m/m\lambda$. 
Then, as shown in the proof of Theorem~\ref{thm:beta-const}, the
gradient of the loss function with respect to $\alpha$ is equal to 
\[
 g_\psi(\alpha)~=~\frac{1}{n}K_{11}v (\alpha,\1_m/m\lambda)+\lambda K_{11}\alpha, 
\]
where the function $v$ depends on $\psi$. On the other hand, the gradient of the loss
function in \eqref{eqn:reduced-KuLSIF} is equal to $K_{11}^{-1}g_\psi(\alpha)$ with 
$\psi(z)=z^2/2$. This fact implies that in KuLSIF, there exists a non-singular matrix 
$C\in \Real^{n\times n}$, which is independent of $\alpha$, such that $Cg_\psi(\alpha)$ is 
identical to the gradient of a function $F(\alpha)$. 
If the condition number of the Hessian matrix of $F(\alpha)$ does not exceed 
$\kappa(H_{\mathrm{KuLSIF}})$, it will be numerically more advantageous to use $F(\alpha)$
as the loss function than KuLSIF. 

Suppose that the $\Real^n$-valued function $C g_\psi(\alpha)$ can be represented as the
gradient of a function $F$, that is, $\nabla F=C g_\psi$. Then, the function $C g_\psi$ is
called {\em integrable} \cite{m03:_geomet_topol_and_physic_secon_edition}. 
What we study in this section is to find $\psi$ such that there exists a non-identity matrix
$C$ such that $C g_\psi(\alpha)$ is integrable. According to 
\emcite{m03:_geomet_topol_and_physic_secon_edition}, the necessary and sufficient condition
of integrability is that the Jacobian matrix of $C g_\psi(\alpha)$ is symmetric. 

The Jacobian matrix of $Cg_\psi(\alpha)$ is equal to 
\[
 \frac{1}{n}CK_{11} D_{\psi,\alpha} K_{11}+\lambda CK_{11}, 
\]
where $D_{\psi,\alpha}$ is the diagonal matrix in which the diagonal elements are given as  
\[
(D_{\psi,\alpha})_{ii}=
\psi''\bigg(\sum_{j=1}^n\alpha_j k(X_i,X_j)+\frac{1}{m\lambda}\sum_{\ell=1}^m
k(X_i,Y_\ell)\bigg),\quad i=1,\ldots,n. 
\]
Let $R$ be the $n$ by $n$ matrix $CK_{11}$, then, the Jacobian matrix is represented as
\[
 M_{\psi,R}(\alpha)~=~\frac{1}{n}R D_{\psi,\alpha} K_{11}+\lambda R. 
\]
\begin{theorem}
 \label{eqn:integrability}
 Let $c$ be a constant value in $\Real$, and the function $\psi$ be second-order
 continuously differentiable. 
 Suppose that the Gram matrix $K_{11}$ is non-singular, and that $K_{11}$ does not have
 zero element. 
 If there exists a non-singular matrix $R\neq c K_{11}$ such that $M_{\psi,R}(\alpha)$
 is symmetric for any $\alpha\in \Real^n$, then, $\psi''$ is a constant function. 
\end{theorem}
The proof may be found in Appendix \ref{appendix:Proof-of-reduction-theorem}. 
 Theorem~\ref{eqn:integrability} guarantees that the condition number of the loss function is
 reducible only when $\psi$ is a quadratic function. Here, multiplying the gradient by a 
 matrix $C$, which is independent of $\alpha$, is allowed as transformation of
 the loss function. For
 other functions $\psi$, the gradient $Cg_{\psi,\alpha}$ cannot be integrable unless 
 $C= cI_n,\ c\in\Real$. 

\begin{remark}
 We summarize the theoretical results on condition numbers. 
 Let $H_{\text{$\psi$-div}}$ be 
 the Hessian matrix \eqref{eqn:Hessian_matrix_f-div} of the M-estimator. 
 Then, the following inequalities hold, 
 \begin{align*}
  \kappa(H_\mathrm{R-KuLSIF})   ~\leq~  &  
  \kappa(K_{11})  ~\leq~  \kappa(H_\mathrm{KuLSIF})  \leq~\kappa(H_\mathrm{KMM}),\\
  \kappa(H_\mathrm{KuLSIF})~=~&
  \sup_{w\in\mathcal{H}}\kappa(H_\mathrm{KuLSIF})  
  ~\leq~
  \sup_{w\in\mathcal{H}}\kappa(H_{\text{$\psi$-div}}). 
 \end{align*}
 Remember that $K_{11}$ is the Hessian matrix of the original (transductive) KMM method,
 and $H_\mathrm{KMM}$
 is its inductive variant. Based on probabilistic evaluation, the inequality 
 \begin{align*}
  \kappa(H_\mathrm{KuLSIF})  ~\leq~ \kappa(H_{\text{$\psi$-div}})
\end{align*}
 will also hold with high probability. 
 Let $H_{\text{$\psi$-KMM}}$ be the Hessian matrix of the loss function 
 $L_{\text{$\psi$-KMM}}$ in Remark~\ref{remark:f-KMM}. 
 Then, we conjecture that 
\begin{align*}
 \kappa(H_{\text{$\psi$-div}})~\leq~\kappa(H_{\text{$\psi$-KMM}})
\end{align*}
 holds in some sense as an extension of the relation between KuLSIF and 
 the inductive variant of KMM. 
 Consequently, R-KuLSIF will be advantageous in numerical computation. 
\end{remark}

\section{Simulation Results}\label{sec:Simulation_Results}
In this section, we experimentally investigate the behavior of the condition numbers.
In the inductive variant of KMM estimator, the Hessian matrix is given by $H_\mathrm{KMM}$ 
defined in \eqref{eqn:Hessian-KMM}. 
In the M-estimator based on $f$-divergence, the Hessian matrix involved in the
optimization problem is given as 
\[
 H=\frac{1}{n}K_{11}D_{\psi,w} K_{11} + \lambda K_{11}\in\Real^{n\times n}. 
\]
For the Kullback-Leibler divergence, we have $\varphi(z)=-\log z$ and
$\psi(z)=-1-\log(-z),\ z<0$, and thus, $\psi'(z)=-1/z$ and $\psi''(z)=1/z^2$ hold for
$z<0$. If the optimal solution provides the true density ratio $w_0$, we obtain
$\psi''(w(x))=\psi''((\psi')^{-1}(w_0(x)))=w_0(x)^2$. Thus, the Hessian matrix is given as 
\begin{align*}
 H_\mathrm{KL}
 ~=~&
 \frac{1}{n}K_{11}{\rm diag}(w_0(X_1)^2,\ldots,w_0(X_n)^2)K_{11}
 +\lambda K_{11}
 \in\Real^{n\times n}. 
\end{align*}
On the other hand, in KuLSIF, 
the Hessian matrix is given by 
$H_\mathrm{KuLSIF}$ defined in
\eqref{eqn:Hessian-KuLSIF}, and the Hessian matrix of R-KuLSIF,
$H_\mathrm{R-KuLSIF}$, is shown in \eqref{eqn:Hessian_RKuLSIF}. 
In examples of Section~\ref{sec:Probabilistic_Evaluation}, we considered the condition
number of a random matrix
\[
 H_\mathrm{RND}~=~
 \frac{1}{n}K_{11}\mathrm{diag}(d_1,\ldots,d_n)K_{11} + \lambda K_{11}
 \in\Real^{n\times n}. 
\]
We use $F_\gamma(d)$ defined in Example~\ref{example:power-decay} 
with various $\gamma$ as the distribution function of $d_1,\ldots, d_n$. 
The condition numbers of Hessian matrices, 
$H_\mathrm{KMM}, H_\mathrm{KL}, H_\mathrm{KuLSIF}, H_\mathrm{R-KuLSIF}$, and $H_\mathrm{RND}$
are numerically compared. In addition, the condition number of $K_{11}$ is also computed. 
In the original transductive KMM estimator defined by \eqref{eqn:kmm-opt-prob}, the condition
number of the loss function is equal to $\kappa(K_{11})$. 
Thus, the convergence rate of numerical optimization in KMM would be approximately governed
by $\kappa(K_{11})$---we need to take the constraints in \eqref{eqn:kmm-opt-prob} into
account to derive more accurate convergence rate of the original KMM. 

The probability densities of $P$ and $Q$ are set to be both the normal distribution on the
$10$-dimensional Euclidean space with the unit variance-covariance matrix $I_{10}$. 
The mean vectors of $P$ and $Q$ are set to $0\times \1_{10}$ and $\mu\times \1_{10}$ with 
$\mu=0.2$ or $\mu=0.5$, respectively. Note that the mean value $\mu$ affects only
$\kappa(H_{\rm KL})$. 
The true density ratio $w_0$ is determined by $P$ and $Q$. 
In the kernel-based estimators, we use the Gaussian kernel with width $\sigma=2$ or
$\sigma=4$. Note that $\sigma=4$ is close to the median of the distance between samples
$\|X_i-X_j\|$; using the median distance as the kernel width is a popular heuristics 
\cite{book:Schoelkopf+Smola:2002}. 
The sample size from $P$ is equal to that from $Q$, that is, $n=m$. 
The regularization parameter $\lambda$ is set to $\lambda_{n,m}=1/(n\wedge m)^{0.9}$ 
which meets the assumption in Theorem \ref{theorem:non-para-bound}. 

Table~\ref{tbl:condition-number-cmp}  
shows the experimental results. 
In each setup, samples $X_1,\ldots,X_n$ and diagonal elements $d_1,\ldots,d_n$ are
randomly generated and the condition number is computed. The table shows the average of the 
condition numbers over 1000 runs. 
As shown in Table \ref{tbl:condition-number-cmp}, the condition number of R-KuLSIF is much
smaller than the other methods for all cases. 
Thus, it is expected that in optimization, the convergence speed of R-KuLSIF is faster 
than the other methods and that R-KuLSIF is robust against numerical degeneracy. 
It will be worthwhile to point out that 
$\kappa(H_{\mathrm{R-KuLSIF}})$ is smaller than $\kappa(K_{11})$. This is because the
identity matrix in $H_{\mathrm{R-KuLSIF}}$ prevents the smallest eigenvalue
from becoming extremely small. The number of $\kappa(H_{\mathrm{RND}})$ is decreasing as 
$\gamma$ tends to large values, and seems to converge to
$\kappa(H_{\mathrm{KuLSIF}})$. This result meets the considerations in 
Remark~\ref{remark:KuLSIF-smallest-cond-num} and 
Example~\ref{example:power-decay}. 


Table~\ref{tbl:optimization-time} shows the average number of iterations and the
average computation time for solving the optimization problems over $50$ runs. 
The probability densities of $P$ and $Q$ are the same as above ones, and 
the mean vector of $Q$ is given as $0.5\times \1_{10}$. 
The numbers of samples are set to 
$(n,m)=(1000,1000),\,(4000,4000)$ or $(6000,6000)$, 
and the regularization parameter is $\lambda=1/(n\wedge m)^{0.9}$. 
The number of $n$ is equal to the number of parameters to be optimized. 
R-KuLSIF, KuLSIF, inductive variant of KMM (KMM), and M-estimator with Kullback-Leibler 
divergence (KL) are compared. In addition, the computation time of solving the linear
equation \eqref{eqn:uLSIF-linear-eq} is also shown as R-KuLSIF(direct). 
The kernel parameter $\sigma$ is determined based on the median of $\|X_i -X_j\|$. 
To solve the optimization problems in the M-estimators and KMM, we used the 
BFGS method implemented in the \texttt{optim} function in R \cite{R}, and 
for R-KuLSIF(direct) we use the \texttt{solve} function. 
The results show that the number of iterations in optimization is highly correlated 
with the condition number of the Hessian matrices in Table~\ref{tbl:condition-number-cmp}. 
Although the practical computational time would depend
on various issues such as stopping rules,
our theoretical results were shown to be in good agreement with
the empirical results.
Thus, the R-KuLSIF would be a stable and computationally efficient density-ratio
estimator. 
We observe that numerical optimization methods such as the quasi-Newton method are 
competitive with numerical algorithms for solving linear equations (such as the LU or
Cholesky methods), especially when the sample size or the number of parameters is large. 
Thus, our results obtained in this paper would be useful in large sample cases---common
situations in practical applications. 


\begin{table*}[tb]
\small 
\caption{Condition numbers of each Hessian matrix.}
\label{tbl:condition-number-cmp}
\centering
 \vspace*{2mm}
 \begin{tabular}{|c|r|r|r|r|r|r|r|r|r|}
  \hline
  & \multicolumn{9}{c|}{kernel width: $\sigma=2$} \\ \cline{2-10}
  & & & & & \multicolumn{2}{c|}{$H_\mathrm{KL}$} &
  \multicolumn{3}{c|}{$H_\mathrm{RND}$}\\ 
  \cline{6-10}
  $n$ & \multicolumn{1}{c|}{$K_{11}$} & 
  $H_\mathrm{R-KuLSIF}$ & $H_\mathrm{KuLSIF}$ & $H_\mathrm{KMM}$ & 
  \multicolumn{1}{c|}{$\mu=0.2$}  &
  \multicolumn{1}{c|}{$\mu=0.5$}  &   
  \multicolumn{1}{c|}{$\gamma=2$} &  
  \multicolumn{1}{c|}{$\gamma=5$} &
  \multicolumn{1}{c|}{$\gamma=10$} \\ \hline
  20  & 1.6e+01&3.8e+00&6.4e+01&2.7e+02&9.0e+01&1.4e+03&1.1e+02&7.4e+01&6.9e+01\\
  50  &	7.1e+01&8.1e+00&5.9e+02&5.1e+03&7.6e+02&4.8e+03&1.1e+03&7.1e+02&6.5e+02\\
 100  &	2.6e+02&1.5e+01&4.1e+03&6.5e+04&5.0e+03&2.7e+04&7.7e+03&5.0e+03&4.5e+03\\
 200  &	1.1e+03&3.0e+01&3.4e+04&1.0e+06&4.2e+04&1.6e+05&6.7e+04&4.2e+04&3.8e+04\\
 300  &	2.9e+03&4.4e+01&1.3e+05&5.7e+06&1.6e+05&5.8e+05&2.5e+05&1.6e+05&1.4e+05\\
 400  &	5.9e+03&5.8e+01&3.4e+05&2.0e+07&4.2e+05&1.5e+06&6.8e+05&4.3e+05&3.8e+05\\
 500  &	1.0e+04&7.3e+01&7.5e+05&5.5e+07&9.2e+05&3.1e+06&1.5e+06&9.4e+05&8.3e+05\\
 \hline
\multicolumn{7}{c}{}\\
  \hline
  & \multicolumn{9}{c|}{kernel width: $\sigma=4$} \\ \cline{2-10}
  & & & & & \multicolumn{2}{c|}{$H_\mathrm{KL}$} &
  \multicolumn{3}{c|}{$H_\mathrm{RND}$}\\
  \cline{6-10}
  $n$ & \multicolumn{1}{c|}{$K_{11}$} & 
  $H_\mathrm{R-KuLSIF}$ & $H_\mathrm{KuLSIF}$ & $H_\mathrm{KMM}$ & 
  \multicolumn{1}{c|}{$\mu=0.2$}  &
  \multicolumn{1}{c|}{$\mu=0.5$}  &   
  \multicolumn{1}{c|}{$\gamma=2$} &  
  \multicolumn{1}{c|}{$\gamma=5$} &
  \multicolumn{1}{c|}{$\gamma=10$} \\ \hline
  20  & 4.3e+02&1.2e+01&5.2e+03&6.3e+04&6.9e+03&2.8e+04&9.9e+03& 6.4e+03&5.7e+03\\
  50  &	4.2e+03&2.8e+01&1.2e+05&3.4e+06&1.6e+05&7.7e+05&2.3e+05& 1.5e+05&1.3e+05\\
 100  &	3.1e+04&5.5e+01&1.7e+06&9.6e+07&2.4e+06&1.2e+07&3.4e+06& 2.2e+06&1.9e+06\\
 200  &	2.6e+05&1.1e+02&2.8e+07&3.1e+09&3.9e+07&2.1e+08&5.6e+07& 3.5e+07&3.2e+07\\
 300  &	1.0e+06&1.6e+02&1.7e+08&2.7e+10&2.3e+08&1.2e+09&3.3e+08& 2.1e+08&1.9e+08\\
 400  &	3.0e+06&2.1e+02&6.3e+08&1.4e+11&8.7e+08&5.0e+09&1.3e+09& 7.9e+08&7.0e+08\\
 500  &	6.5e+06&2.7e+02&1.7e+09&4.6e+11&2.4e+09&1.3e+10&3.4e+09& 2.2e+09&1.9e+09\\\hline
 \end{tabular}
\end{table*}


\begin{table*}
\centering
\small
 \caption{Averages of the computation time and the number of iterations in
 the BFGS method over 50 runs.}
\label{tbl:optimization-time}
\centering
 \vspace*{2mm}
 \begin{tabular}{|l||r|r||r|r||r|r|}
  \hline
  & \multicolumn{2}{|c||}{$n=1000,\ m=1000$}
  & \multicolumn{2}{|c||}{$n=4000,\ m=4000$}
  & \multicolumn{2}{|c|}{$n=6000,\ m=6000$} \\\hline
  \multicolumn{1}{|c||}{Estimator}  & 
  \begin{tabular}{@{}c@{}}
    Comput.\\ time (sec.)
  \end{tabular}&
  \begin{tabular}{@{}c@{}}
    Number of\\iterations
  \end{tabular}&
  \begin{tabular}{@{}c@{}}
    Comput.\\ time (sec.)
  \end{tabular}&
  \begin{tabular}{@{}c@{}}
    Number of\\iterations
  \end{tabular}&
    \begin{tabular}{@{}c@{}}
    Comput.\\ time (sec.)
  \end{tabular}&
  \begin{tabular}{@{}c@{}}
    Number of\\iterations
  \end{tabular}\\  \hline
  R-KuLSIF         & 1.44 & 23.02 &   34.94& 29.98&  71.69& 30.74\\	
  KuLSIF           & 2.25 & 38.36 &   53.93& 48.76& 107.79& 47.32\\	
  KMM              &51.83 &453.68 &  591.44&400.74&1091.69&373.08\\	
  KL               &27.63 &329.06 & 1180.72&634.32&2718.89&669.20\\	
  R-KuLSIF(direct) & 0.46 &    -- &   28.85&    --&  87.06&    --\\ \hline 
  \multicolumn{7}{c}{(CPU: Xeon X5482, 3.20GHz, Memory: 32GB, OS: Linux 2.6.18)}
  \end{tabular}
\end{table*}

\section{Conclusions}\label{sec:Conclusion}
We considered the problem of estimating the ratio of two probability densities 
and investigated theoretical properties of the kernel least-squares estimator called
KuLSIF. 
We studied the condition number of Hessian matrices, and showed that KuLSIF has a smaller
condition number than the other methods. 
Since the condition number determines the convergence rate of optimization and the
numerical stability, KuLSIF will have a preferable numerical properties to the other
methods. 
We further showed that R-KuLSIF, which is an alternative formulation of KuLSIF, possesses
an even smaller condition number. 


Density ratio estimation could provide
new approaches to various machine learning problems
including covariate shift adaptation 
\cite{NIPS2006_915,NIPS:Sugiyama+etal:2008,kanamori09:_least_squar_approac_to_direc_impor_estim,bickel09:_discr_learn_under_covar_shift},
outlier detection 
\cite{ICDM:Hido+etal:2008},
and feature selection 
\cite{FSDM:Suzuki+etal:2008}. 
Based on the theoretical guidance given in this paper, we will develop practical
algorithms for a wide-range of applications in the future work.

\appendix
\section{Proof of Theorem~\ref{theorem:non-para-bound}}
\label{appendix:Proof_Convergence_Theorem}
Let us define the bracketing entropy of the set of functions. 
For distribution function $P$, define the $L_2$ metric
\[
 \|g\|_P~=~\left(\int |g|^2 dP\right)^{1/2}, 
\]
and let $L_2(P)$ be the metric space defined by this distance. For any fixed $\delta>0$,
a covering for function class $\mathcal{S}$ using the metric $L_2(P)$ is a collection of
functions which allow $\mathcal{S}$ to be covered using $L_2(P)$ balls of radius $\delta$
centered at these functions. Let $N_B(\delta,\mathcal{S},P)$ be the smallest value of $N$
for which there exist pairs of functions 
$\{(g_j^L,g_j^U)\in L_2(P)\times L_2(P)~|~j=1,\ldots,N\}$ such that 
$\|g_j^L-g_j^U\|_P \leq \delta$, and such that for each $s\in\mathcal{S}$, there exists $j$
such that $g_j^L\leq s\leq g_j^U$. Then, $H_B(\delta,\mathcal{S},P)=\log
N_B(\delta,\mathcal{S},P)$ is called the {\it bracketing entropy} of $\mathcal{S}$
\cite{Book:VanDeGeer:EmpiricalProcess}.

Let $\mathcal{H}$ be the RKHS endowed with the Gaussian kernels,
$k(x,y)=e^{-\|x-y\|^2/2\sigma^2}$. The norm and inner product on $\mathcal{H}$ are denoted
by $\|\cdot\|_{\mathcal{H}}$ and $\<\cdot,\cdot\>_{\mathcal{H}}$, respectively. 
 Let $\|\cdot\|_\infty$ be the infinity norm. 
 For $w\in\mathcal{H}$, we have $\|w\|_P\leq \|w\|_\infty\leq \|w\|_{\mathcal{H}}$, 
 because for any $x\in \ZC$, the inequalities 
 \[
 |w(x)|=|\<w,k(\cdot,x)\>_{\mathcal{H}}|
 \leq\|w\|_{\mathcal{H}}\,\sup_x k(x,x)
 =\|w\|_{\mathcal{H}}
 \]
 holds.
The set $\ZC$, which is the domain of functions in $\HC$, is assumed to be compact. Let
$\mathcal{G}=\{v^2~|~v\in\mathcal{H}\}$. Let $\HC_M$ and $\GC_M$ be  
\begin{align}
 \HC_M ~=~& \{v\in \HC~|~\|v\|_\mathcal{H}< M\},\nonumber\\
 \GC_M ~=~& \{v^2~|~v\in \HC_{\sqrt{M}}\}
 ~=~ \{g\in \GC~|~J(g)< M\}, \label{eqn:def-G_M}
\end{align}
where $J(g)$ is a measure of complexity defined as
\begin{align*}
 J(g)=\inf\,\{\|v\|_{\mathcal{H}}^2~|~v\in\HC,\, v^2=g\}.
\end{align*}
It is straightforward to verify the second equality of \eqref{eqn:def-G_M}. 
According to \emcite{zhou02:_cover_number_in_learn_theor}, the bracketing entropy 
of $\mathcal{H}_M$ satisfies, for infinitesimally small $\gamma>0$, the condition 
\begin{align}
 H_B(\delta,\HC_M,P) ~=~O\left(\frac{M}{\delta}\right)^\gamma. 
 \label{eqn:bracketing_entropy_Gaussian}
\end{align}
More precisely, \emcite{zhou02:_cover_number_in_learn_theor} have proved that 
the entropy number with the supremum norm is bounded above by $O((M/\delta)^\gamma)$.  
In addition, the bracketing entropy  
$H_B(\delta,\HC_M,P)$ is bounded above by the entropy number with the supremum norm due to
Lemma 2.1 in \emcite{Book:VanDeGeer:EmpiricalProcess}. 

The following proposition is crucial to prove the convergence property of KuLSIF.  
\begin{proposition}
[Lemma 5.14 in \emcite{Book:VanDeGeer:EmpiricalProcess}]
\label{theorem:Geer's-book}
Let a map $I(g)$ be a measure of complexity of $g\in\GC$, where $I$ is a non-negative
functional on $\GC$ and $I(g_0)<\infty$. Then, we define $\GC_M = \{g\in \GC~|~I(g)< M\}$
 satisfying $\GC=\cup_{M\geq 1}\GC_M$. 
Suppose that there exist $c_0>0$ and $0<\gamma<2$ such that 
\begin{align*}
 \sup_{g\in \mathcal{G}_M}\|g-g_0\|_P &\leq c_0 M,\quad
 \sup_{\substack{g\in \mathcal{G}_M\\ \|g-g_0\|_P\leq\delta}}
 \|g-g_0\|_\infty \leq c_0 M,\quad\ \text{for all $\delta>0$,}
\end{align*}
and that $H_B(\delta,\GC_M,P)=O\left(M/\delta\right)^\gamma$. 
 Then, we have 
 \begin{align*}
  \sup_{g\in \mathcal{G}}
  \frac{\bigg|\displaystyle\int (g-g_0)d(P-P_n)\bigg|}{D(g)} = O_p(1), 
 \end{align*}
where $D(g)$ is defined as 
\begin{align*}
 D(g)=\frac{\|g-g_0\|_P^{1-\gamma/2}I(g)^{\gamma/2}}{\sqrt{n}}
 \vee \frac{I(g)}{n^{2/(2+\gamma)}}
\end{align*}
and $a \vee b$ denotes $\max\{a, b\}$. 
\end{proposition}
 
We use Proposition~\ref{theorem:Geer's-book} to derive an upper bound of 
$\int(\what-w_0)d(Q-Q_m)$ and $\int(\what^2-w_0^2)d(P-P_n)$. 

\begin{lemma}
 The bracketing entropy of $\mathcal{G}_M$ is bounded above as 
\[
 H_B(\delta,\GC_M,P) ~=~O\left(\frac{M}{\delta}\right)^\gamma 
\]
 for any small $\gamma>0$. 
\end{lemma}
\begin{proof}
 Let $v_1^L,v_1^U,v_2^L,v_2^U,\ldots,v_N^L,v_N^U\in L_2(P)$ be coverings of
 $\HC_{\sqrt{M}}$ in the sense of bracketing, such that $\|v_i^L-v_i^U\|_P\leq \delta$
 holds for $i=1,\ldots,N$. 
 We can choose these functions such that  
 $\|v_i^{L(U)}\|_\infty \leq \sqrt{M}$ 
 is satisfied for
 all $i=1,\ldots,N$, since for any $v\in \mathcal{H}_{\sqrt{M}}$, the inequality 
 $\|v\|_\infty \leq \|v\|_{\mathcal{H}}<\sqrt{M}$ holds. 
 For example, replace $v_i^{L(U)}$ with 
 $\min\{\sqrt{M}\,,\max\{-\sqrt{M}, v_i^{L(U)}\}\}\in L_2(P)$. 
 Let $\bar{v}_i^L$ and $\bar{v}_i^U$ be 
 \begin{align*}
  \bar{v}_i^L(x)
  &=
  \begin{cases}
   (v_i^L(x))^2 & v_i^L(x)\geq 0,\\
   (v_i^U(x))^2 & v_i^U(x)\leq 0,\\
    0           & v_i^L(x)<0<v_i^L(x),
  \end{cases}\\
  \bar{v}_i^U&=\max\{(v_i^L)^2,\,(v_i^U)^2\}, 
 \end{align*}
 for $i=1,\ldots,N$. 
 Then, $\bar{v}_i^L\leq \bar{v}_i^U$ holds. 
 Moreover, for any $v\in \mathcal{H}_{\sqrt{M}}$ satisfying $v_i^L\leq v\leq v_i^U$, we have 
 $\bar{v}_i^L\leq v^2\leq \bar{v}_i^U$. 
 By definition, we also have 
 \begin{align*}
  0
 ~\leq~ &
  \bar{v}_i^U(x)-\bar{v}_i^L(x)
 ~\leq~
  \max\{|v_i^U(x)^2-v_i^L(x)^2|,\ |v_i^U(x)-v_i^L(x)|^2\}\\
  ~\leq~ &
  (|v_i^U(x)|+|v_i^L(x)|)\cdot|v_i^U(x)-v_i^L(x)| 
  ~\leq~ 2\sqrt{M}|v_i^U(x)-v_i^L(x)|, 
 \end{align*}
 and thus, $\|\bar{v}_i^U-\bar{v}_i^L\|_P\leq 2\sqrt{M}\|v_i^U-v_i^L\|_P$ holds.  
 Due to \eqref{eqn:bracketing_entropy_Gaussian}, we obtain
\[
 H_B(2\sqrt{M}\delta,\GC_M,P) \leq H_B(\delta,\HC_{\sqrt{M}},P)
 =O\left(\frac{\sqrt{M}}{\delta}\right)^\gamma. 
\]
Hence, $H_B(\delta,\GC_M,P)=O\left(M/\delta\right)^\gamma$ holds. 
\end{proof} 

\begin{lemma}
\label{lemma:ULLN}
 Assume the condition of Theorem~\ref{theorem:non-para-bound}. 
 Then, for the KuLSIF estimator $\what$, we have 
 \begin{align*}
\bigg|\int (\what-w_0)d(Q-Q_m)\bigg|
 ~=~& O_p\left(
 \frac{\|w_0-\what\|_P^{1-\gamma/2}\|\what\|_{\mathcal{H}}^{\gamma/2}}{\sqrt{m}}
  \vee \frac{\|\what\|_{\mathcal{H}}}{m^{2/(2+\gamma)}}\right),  \\
  \bigg|\int (\what^2-w_0^2)d(P-P_n)\bigg|
 ~=~& 
  O_p\left(\frac{\|\what-w_0\|_P^{1-\gamma/2}
  (1+\|\what\|_{\mathcal{H}})^{1+\gamma/2}}{\sqrt{n}} 
  \vee 
  \frac{\|\what\|_{\mathcal{H}}^2}{n^{2/(2+\gamma)}}\right), 
 \end{align*}
 where $\gamma>0$ is an infinitesimally small value. 
\end{lemma}
\begin{proof}
 There exists $c_0>0$ such that 
 \begin{align}
  &\sup_{w\in \mathcal{H}_M} \|w-w_0\|_P ~\leq~ c_0 M,\quad
  \sup_{\substack{w\in \mathcal{H}_M\\ \|w-w_0\|_P\leq \delta}}
  \|w-w_0\|_\infty  ~\leq~ c_0 M, \label{eqn:condition-conv-v}\\
  &\sup_{g\in \mathcal{G}_M} \|g-w_0^2\|_P ~\leq~ c_0 M,\quad
  \sup_{\substack{g\in \mathcal{G}_M\\ \|g-w_0^2\|_P\leq \delta}} 
  \|g-w_0^2\|_\infty  ~\leq~ c_0 M. \label{eqn:condition-conv-f(v)}
 \end{align}
The inequalities in \eqref{eqn:condition-conv-f(v)} are derived as follows. 
 For $g\in \mathcal{G}_M$, there exists $v\in\mathcal{H}$ such that $v^2=g$ and
 $\|v\|_{\mathcal{H}}^2<M$, and then, we have
 \begin{align*}
  \|g-w_0^2\|_P
&  ~\leq~ \|g-w_0^2\|_\infty 
  ~\leq~ \|v\|_\infty^2 + \|w_0\|_\infty^2\\
&  ~\leq~ \|v\|_{\mathcal{H}}^2 + \|w_0\|_\infty^2  
  ~\leq~ M + \|w_0\|_\infty^2
  ~\leq~ c_0 M,\ \  (M\geq 1). 
 \end{align*}
 In the same way,
 \eqref{eqn:condition-conv-v} also holds. 
 Therefore, due to 
 Proposition~\ref{theorem:Geer's-book} 
 and \eqref{eqn:condition-conv-v}, we have 
 \begin{align*}
  \sup_{w\in\mathcal{H}}\ \frac{\displaystyle\left|\int(w_0-w)d(Q-Q_m)\right|}{D(w)} ~=~O_p(1),
 \end{align*}
 where $D(w)$ is defined as
 \begin{align*}
  D(w) ~=~
  \frac{\|w_0-w\|_P^{1-\gamma/2}\|w\|_{\mathcal{H}}^{\gamma/2}}{\sqrt{m}}
  \vee \frac{\|w\|_{\mathcal{H}}}{m^{2/(2+\gamma)}}. 
  \end{align*}
 In the same way, we have
 \begin{align*}
  \sup_{w\in\mathcal{H}}\ 
  \frac{\displaystyle\bigg|\int (w^2-w_0^2)d(P-P_n)\bigg|}{E(w)}
  ~=~O_p(1), 
 \end{align*}
 where $E(w)$ is defined as
 \begin{align*}
  E(w) ~=~
  \frac{\|w^2-w_0^2\|_P^{1-\gamma/2}J(w^2)^{\gamma/2}}{\sqrt{n}}
  \vee  \frac{J(w^2)}{n^{2/(2+\gamma)}}. 
  \end{align*}
 Note that 
 $\|w^2-w_0^2\|_P ~\leq~ 
 (\|w_0\|_\infty+\|w\|_\mathcal{H})\|w-w_0\|_P
 ~=~ O((1+\|w\|_\mathcal{H})\|w-w_0\|_P)$
and 
 $J(w^2)\leq \|w\|_{\mathcal{H}}^2$. Then, we obtain
 \begin{align*}
  E(w) 
  ~\leq~&
  \frac{\|w-w_0\|_P^{1-\gamma/2}
  (1+\|w\|_{\mathcal{H}})^{1+\gamma/2}}{\sqrt{n}}
  \vee 
  \frac{\|w\|_{\mathcal{H}}^2}{n^{2/(2+\gamma)}}. 
  \end{align*}
\end{proof}

Now we show the proof of Theorem~\ref{theorem:non-para-bound}. 
\begin{proof}
The estimator $\what$ satisfies the inequality
\begin{align*}
 \frac{1}{2}\int \what^2 dP_n-\int\what dQ_m+\frac{\lambda}{2}\|\what\|_{\mathcal{H}}^2
 ~\leq~  
 \frac{1}{2}\int w_0^2 dP_n-\int w_0 dQ_m+\frac{\lambda}{2}\|w_0\|_{\mathcal{H}}^2.
\end{align*}
Then, we have
\begin{align*}
 \frac{1}{2}\|\what-w_0\|^2_{P}
 ~=~ &
 \int (w_0-\what)dQ+\frac{1}{2}\int(\what^2-w_0^2)dP\\
 ~\leq~ &
 \int (w_0-\what)dQ+\frac{1}{2}\int(\what^2-w_0^2)dP\\
 +&\int (\what -w_0)dQ_m + 
 \frac{1}{2}\int (w_0^2-\what^2)dP_n
 +\frac{\lambda}{2}\|w_0\|_{\mathcal{H}}^2-\frac{\lambda}{2}\|\what\|_{\mathcal{H}}^2 .
 \end{align*}
 As a result, we have
 \begin{align*}
  &\frac{1}{2}\|\what-w_0\|^2_{P}+\frac{\lambda}{2}\|\what\|_{\mathcal{H}}^2 \\
 ~\leq~ &    
  \bigg|\int (\what-w_0)d(Q-Q_m)\bigg|+
  \frac{1}{2}\bigg|\int (\what^2-w_0^2)d(P-P_n)\bigg|
  +\frac{\lambda}{2}\|w_0\|_{\mathcal{H}}^2\\
 ~\leq~ &      
  \frac{\lambda}{2}\|w_0\|_{\mathcal{H}}^2
  +O_p\left(
  \frac{\|w_0-\what\|_P^{1-\gamma/2}(1+\|\hatw\|_{\mathcal{H}})^{1+\gamma/2}}{\sqrt{n\wedge m}}
  \vee \frac{(1+\|\what\|_{\mathcal{H}})^2}{(n\wedge m)^{2/(2+\gamma)}}
  \right), 
 \end{align*}
 where Lemma~\ref{lemma:ULLN} is used. 

We need to study three possibilities: 
\begin{align}
 &
 \frac{1}{2}\|w_0-\what\|_{P}^2+\frac{\lambda}{2}\|\what\|^2_{\cal H}  
 ~\leq~ O_p(\lambda),\label{eqn:case-1}\\
&
 \frac{1}{2}\|w_0-\what\|_{P}^2+\frac{\lambda}{2}\|\what\|^2_{\cal H}  
 ~\leq~ O_p\left(
 \frac{\|w_0-\what\|_{P}^{1-\gamma/2} 
 (1+\|\what\|_{\mathcal{H}})^{1+\gamma/2}}{\sqrt{n\wedge m}}
 \right),
\label{eqn:case-2}\\
&
 \frac{1}{2}\|w_0-\what\|_{P}^2+\frac{\lambda}{2}\|\what\|^2_{\cal H}
 ~\leq~ O_p\left( \frac{(1+\|\what\|_\mathcal{H})^2} 
 {(n\wedge m)^{2/(2+\gamma)}}\right).\label{eqn:case-3}
\end{align}
One of the above inequalities should be satisfied. We study each inequality below. 

\noindent
{\bf Case \eqref{eqn:case-1}}: we have
\begin{align*}
 \frac{1}{2}\|w_0-\what\|_{P}^2 \leq O_p(\lambda),\ \ \ 
 \frac{\lambda}{2}\|\what\|^2_{\cal H}  ~\leq~ O_p(\lambda), 
\end{align*} 
 and hence the inequalities 
 $\|w_0-\what\|_{P} \leq O_p(\lambda^{1/2})$ and $\|\what\|_{\cal H} \leq O_p(1)$ 
 hold. 

\noindent 
{\bf Case \eqref{eqn:case-2}}: we have
\begin{align*}
 \|w_0-\what\|_{P}^2 
 &~\leq~ 
 O_p\left(
 \frac{\|w_0-\what\|_{P}^{1-\gamma/2}(1+\|\what\|_{\mathcal{H}})^{1+\gamma/2}}
 {(n\wedge m)^{1/2}} \right),\\
 \lambda\|\what\|_{\cal H}^2 
  &~\leq~ 
 O_p\left(
 \frac{\|w_0-\what\|_{P}^{1-\gamma/2}(1+\|\what\|_{\mathcal{H}})^{1+\gamma/2}}
 {(n\wedge m)^{1/2}} \right). 
\end{align*}
The first inequality provides
 \begin{align*}
 \|w_0-\what\|_{P} ~\leq~ 
  O_p\left(
  \frac{1+\|\widehat{w}\|_{\mathcal{H}}}{(n\wedge m)^{1/(2+\gamma)}}
  \right). 
 \end{align*}
 Thus, the second inequality leads to
\begin{align*}
 \lambda\|\what\|_{\mathcal{H}}^2
 ~\leq~&
  O_p\left(
 \frac{\|w_0-\what\|_{P}^{1-\gamma/2}(1+\|\what\|_{\mathcal{H}})^{1+\gamma/2}}
 {(n\wedge m)^{1/2}} 
 \right)\\
 ~\leq~&
  O_p\left(
 \left(\frac{1+\|\widehat{w}\|_{\mathcal{H}}}{(n\wedge m)^{1/(2+\gamma)}}\right)^{1-\gamma/2}
 \frac{(1+\|\what\|_{\mathcal{H}})^{1+\gamma/2}}
 {(n\wedge m)^{1/2}} 
 \right)\\
  ~=~&
 O_p\left(
 \frac{(1+\|\what\|_{\mathcal{H}})^2}{(n\wedge m)^{2/(2+\gamma)}}
 \right). 
\end{align*}
 Hence, we have
 \begin{align*}
  \|\what\|_{\mathcal{H}}~\leq~
  O_p\left(\frac{1}{\lambda^{1/2}(n\wedge m)^{1/(2+\gamma)}}\right)~=~o_p(1) 
 \end{align*}
 for infinitesimally small $\gamma>0$. Then, we obtain
 \begin{align*}
  \|w_0-\what\|_{P} ~\leq~
  O_p\left(\frac{1}{(n\wedge m)^{1/(2+\gamma)}}\right)~\leq~O_p(\lambda^{1/2}). 
 \end{align*}

\noindent
{\bf Case \eqref{eqn:case-3}}: we have
\begin{align*}
 \|w_0-\what\|_{P}^2 ~\leq~ 
 O_p\left( \frac{(1+\|\what\|_\mathcal{H})^2} 
 {(n\wedge m)^{2/(2+\gamma)}}\right),
 \qquad 
 \lambda\|\what\|_{\cal H}^2 ~\leq~ 
 O_p\left( \frac{(1+\|\what\|_\mathcal{H})^2} 
 {(n\wedge m)^{2/(2+\gamma)}}\right).
\end{align*}
 Then, as shown in the case \eqref{eqn:case-2}, we have
 $\|\what\|_{\mathcal{H}}=o_p(1)$. Hence, we obtain
 \begin{align*}
  \|w_0-\what\|_{P} 
  ~\leq~ 
  O_p\left(\frac{1}{(n\wedge m)^{1/(2+\gamma)}}\right)
  ~\leq~ O_p(\lambda^{1/2}). 
 \end{align*}
\end{proof}

\section{Leave-one-out Cross-validation of KuLSIF}
\label{appendix:LOOCV}
The procedure to compute 
the leave-one-out cross-validation score of KuLSIF is presented here. 
Let $K_{11}^{(\ell)}\in \Real^{(n-1)\times(n-1)}$ and
$K_{12}^{(\ell)}=K_{21}^{(\ell)\top}\in\Real^{(n-1)\times (m-1)}$ be the Gram matrices of
samples except $x_\ell$ and $y_\ell$, respectively.
According to Theorem \ref{thm:beta-const}, the
estimated parameters $\widetilde{\alpha}^{(\ell)}$ and $\widetilde{\beta}^{(\ell)}$ of 
\[
\what^{(\ell)}(z)=\sum_{i\neq \ell}\alpha_i k(z,X_i)+\sum_{j\neq \ell}\beta_j k(z,Y_j)
\]
is equal to 
\begin{align*}
 \widetilde{\alpha}^{(\ell)}~=~-\frac{1}{(m-1)\lambda}(K_{11}^{(\ell)}+(n-1)\lambda I_{n-1})^{-1}
 K_{12}^{(\ell)}\1_{m-1},\quad
 \widetilde{\beta}^{(\ell)}~=~\frac{1}{(m-1)\lambda}\1_{m-1}, 
\end{align*}
where $I_{n-1}$ denotes the $(n-1)$ by $(n-1)$ identity matrix. Hence, the parameter 
$\widetilde{\alpha}^{(\ell)}$ is the solution of the following convex  quadratic problem, 
\begin{align}
 \min_\alpha\ \frac{1}{2}\alpha^\top (K_{11}^{(\ell)}+(n-1)\lambda I_{n-1})\alpha
 +\frac{1}{(m-1)\lambda}\1_{m-1}^\top K_{21}^{(\ell)}\alpha,\ \ \ \alpha\in\Real^{n-1}. 
 \label{eqn:loocv-alpha-est}
\end{align}
The same solution can be obtained by solving 
\begin{align}
 \begin{array}{l}
  \displaystyle
 \min_\alpha\ \frac{1}{2}\alpha^\top (K_{11}+(n-1)\lambda I_n)\alpha
 +\frac{1}{(m-1)\lambda}(\1_{m}-\e_{m,\ell})^\top K_{21}\alpha,\\
  \displaystyle
  \qquad \st\ \alpha\in\Real^{n},\  \alpha_\ell=0, 
 \end{array}
 \label{eqn:loocv-alpha-est}
\end{align}
where $\e_{m,\ell}\in\Real^m$ is the standard unit vector with only the $\ell$-th
component being 1. The optimal solution of \eqref{eqn:loocv-alpha-est} denoted by
$\alpha^{(\ell)}$ is equal to 
\begin{align*}
 \alpha^{(\ell)}=
 (K_{11}+(n-1)\lambda I_n)^{-1}\left(-\frac{1}{(m-1)\lambda}
 K_{12}(\1_m-\e_{m,\ell})-c_\ell \e_{n,\ell}\right), 
\end{align*}
where $c_\ell$ is determined so that $\alpha^{(\ell)}_\ell=0$. 
The estimator $\widetilde{\alpha}^{(\ell)}\in\Real^{n-1}$ is equal to the $(n-1)$-dimensional
vector consisting of $\alpha^{(\ell)}$ except the $\ell$-th component,
i.e.,
$\widetilde{\alpha}^{(\ell)}=(\alpha^{(\ell)}_1,\ldots,\alpha^{(\ell)}_{\ell-1},
\alpha^{(\ell)}_{\ell+1},\ldots,\alpha^{(\ell)}_{n})^\top$. 

The parameters of the leave-one-out estimator, 
\[
A=(\alpha^{(1)},\ldots,\alpha^{(n\wedge m)})\in\Real^{n\times (n\wedge m)},\quad
B=(\beta^{(1)},\ldots,\beta^{(n\wedge m)})\in\Real^{m\times (n\wedge m)}
\]
also have analytic expressions. 
Let $G\in\Real^{n\times n}$ be $G=(K_{11}+(n-1)\lambda I_n)^{-1}$, and 
$E\in\Real^{m\times (n\wedge m)}$ be the matrix defined as 
\[
 E_{ij}=\begin{cases}
	 1& i\neq j,\\
	 0& i=j.
	\end{cases} 
\]
Let $S\in \Real^{n\times (n\wedge m)}$ be 
\[
 S=-\frac{1}{(m-1)\lambda}K_{12}E, 
\]
and $T\in \Real^{n\times (n\wedge m)}$ be 
\[
 T_{ij}=\begin{cases}
	 \displaystyle	 \frac{(GS)_{ii}}{G_{ii}}& i=j,\\
	 \displaystyle	 0        & i\neq j.
	\end{cases}
\]
Then, we obtain 
\begin{align*}
 A=G(S-T),\quad 
 B=\frac{1}{(m-1)\lambda}E. 
\end{align*}
Let $K_X\in \Real^{(n\wedge m)\times (n+m)}$ be the sub-matrix of $(K_{11} K_{12})$ formed
by the first $n\wedge m$ rows and all columns. Similarly, let 
$K_Y\in \Real^{(n\wedge m)\times (n+m)}$ be the sub-matrix of $(K_{21} K_{22})$ formed by
the first $n\wedge m$ rows  and all columns. 
Let the product $U*U'$ be the element-wise 
multiplication of matrices $U$ and $U'$ of the same size, i.e., the $(i,j)$ element is
given by $U_{ij}U'_{ij}$. Then, we have
\begin{align*}
 \what_X&=(\what^{(1)}(X_1),\ldots,\what^{(n\wedge m)}(X_{n\wedge m}))^\top
 =(K_X * (A^\top\ B^\top))\1_{n+m},\\ 
 \what_Y&=(\what^{(1)}(Y_1),\ldots,\what^{(n\wedge m)}(Y_{n\wedge m}))^\top
 =(K_Y * (A^\top\ B^\top))\1_{n+m}, \\
 \what_{X+}&=
 (\what^{(1)}_+(X_1),\ldots,\what^{(n\wedge m)}_+(X_{n\wedge m}))^\top =\max\{\what_X,0\},\\
 \what_{Y+}&=
 (\what^{(1)}_+(Y_1),\ldots,\what^{(n\wedge m)}_+(Y_{n\wedge m}))^\top
 =\max\{\what_Y,0\},
\end{align*}
where the max operation for a vector is applied in the element-wise manner. As a result,
LOOCV 
\eqref{eqn:LOOCV} is equal to
\[
 {\rm LOOCV}~=~\frac{1}{n\wedge m}
 \left\{\frac{1}{2}\what_{X+}^\top\what_{X+}-\1_{n\wedge m}^\top\what_{Y+}\right\}. 
\]

\section{Proof of Eq.\,\eqref{eqn:trade-off_accuracy_speed}}
\label{appendix:condition-number-equality}
Let $\kappa(A)$ be the condition number of the symmetric positive definite matrix $A$,
then 
we prove that the following equality
\begin{align*}
 \min_{S:\kappa(S)\leq C}\kappa(SAS) ~=~\max\left\{\frac{\kappa(A)}{C^2},\ 1\right\}
\end{align*}
holds. 
The same equality holds for the condition number defined through singular values for 
non-symmetric matrices. We prove the case that $S$ is a symmetric matrix for simplicity. 
Note that $\kappa(S^2)=\kappa(S)^2$ and $\kappa(S)=\kappa(S^{-1})$, thus we obtain 
Eq. \eqref{eqn:trade-off_accuracy_speed}, i.e.
\begin{align*}
 \min_{S:\kappa(S)\leq C}\kappa(S^{-1/2}AS^{-1/2}) ~=~\max\left\{\frac{\kappa(A)}{C},\
 1\right\}. 
\end{align*}
\begin{proof}
 First, we prove 
 $\min_{S:\kappa(S)\leq C}\kappa(SAS) ~\geq ~\max\{\frac{\kappa(A)}{C^2},\ 1\}$. 

 The matrix $A$ is symmetric positive definite, thus, there exists an orthogonal matrix
 $Q$ and a diagonal matrix $\Lambda=\mathrm{diag}(\lambda_1,\ldots,\lambda_n)$ such that 
 $A=Q\Lambda Q^\top$. 
 The eigenvalues are arranged in the decreasing order, i.e.,
 $\lambda_1\geq \lambda_2\geq\cdots\geq \lambda_n>0$. 
 In the similar way, let $S$ be $PDP^\top$, where $P$ is an orthogonal matrix and 
 $D=\mathrm{diag}(d_1,\ldots,d_n)$ is a diagonal matrix such that 
 $d_1\geq d_2\geq \cdots \geq d_n>0$ and $d_1/d_n\leq C$. Hence, 
 \begin{align*}
  \kappa(SAS)  = \kappa(PDP^\top Q\Lambda Q^\top PDP^\top)
  = \kappa(DP^\top Q\Lambda Q^\top PD). 
 \end{align*}
 Let $Q^\top P$ be $R^\top$ which is also an orthogonal matrix. 
 The maximum eigenvalue of $DR \Lambda R^\top D$ is given as 
 \begin{align*}
  \max_{\|\x\|=1}\x^\top DR\Lambda R^\top D\x. 
 \end{align*}
 Let $R=(\r_1,\ldots,\r_n)$, where $\r_i\in\Real^n$, and we choose $\x_1$ such that 
 $\r_i^\top D\x_1=0$ for $i=2,\ldots,n$ and $\|\x_1\|=1$. Then, 
  \begin{align*}
  \max_{\|\x\|=1}\x^\top D R\Lambda R^\top D\x
   ~\geq~
   \x_1^\top D R\Lambda R^\top D\x_1
   ~=~ 
   \lambda_1(\x_1^\top D \r_1)^2. 
 \end{align*}
 From the assumption on $\x_1$, $D\x_1$ is represented as $c\r_1$ for some $c$, and we have
 $(\x_1^\top D \r_1)^2 = c^2=\x_1^\top D^2\x_1\geq d_n^2$. 
 Hence, we have
 \begin{align*}
  \max_{\|\x\|=1}\x^\top SAS\x \geq \lambda_1d_n^2. 
 \end{align*}
 
 On the other hand, the minimum eigenvalue of $DR \Lambda R^\top D$ is given as 
 \begin{align*}
  \min_{\|\x\|=1}\x^\top DR\Lambda R^\top D\x. 
 \end{align*}
 We choose $\x_n$ such that 
 $\r_i^\top D\x_n=0$ for $i=1,\ldots,n-1$. Then, 
  \begin{align*}
  \min_{\|\x\|=1} \x^\top D R\Lambda R^\top D\x
   ~\leq~&
   \x_n^\top D R\Lambda R^\top D\x_n\\
   ~=~ &
   \lambda_n(\x_n^\top D \r_n)^2 \\
   ~\leq~ &
   \lambda_n \x_n^\top D^2\x_n\quad \text{(Schwarz inequality)}\\
   ~\leq~ &
   \lambda_n d_1^2. 
  \end{align*}
As a result, the condition number of $SAS$ is bounded below as
\begin{align*}
 \kappa(SAS) 
 \geq \frac{\lambda_1d_n^2}{\lambda_nd_1^2}
 =  \frac{\kappa(A)}{(d_1/d_n)^2}
 \geq  \frac{\kappa(A)}{C^2}. 
\end{align*}

 Next, we prove
 $\min_{S:\kappa(S)\leq C}\kappa(SAS)\leq \max\{\frac{\kappa(A)}{C^2},\ 1\}$. 
 If $\kappa(A)\leq C^2$, the inequality $\min_{S:\kappa(S)\leq C}\kappa(SAS)=1$ holds, 
 because we can choose $S=A^{-1/2}$. Then, we prove 
 $\min_{S:\kappa(S)\leq C}\kappa(SAS)\leq \frac{\kappa(A)}{C^2}$, if 
 $1\leq C^2\leq \kappa(A)$ is satisfied. 

 Let $S=Q\Gamma Q^\top$ with $\Gamma$ be a diagonal matrix
 $\mathrm{diag}(\gamma_1,\ldots,\gamma_n)$,
 then $\kappa(SAS)=\kappa(\mathrm{diag}(\gamma_1^2\lambda_1,\ldots,\gamma_n^2\lambda_n))$
 holds. 
 Let $\gamma_1=1$ and $\gamma_n=C$. 
 Since $1\leq C^2\leq \kappa(A)=\lambda_1/\lambda_n$ holds, 
 for $k=2,\ldots,n-1$ we have 
 \begin{align*}  
  1\leq \min\bigg\{C,\,\sqrt{\frac{\lambda_1}{\lambda_k}}\bigg\}, \qquad
  C\sqrt{\frac{\lambda_n}{\lambda_k}}\leq
  \min\bigg\{C,\,\sqrt{\frac{\lambda_1}{\lambda_k}}\bigg\}
 \end{align*}
 and thus, we obtain
 \begin{align*}  
  \max\bigg\{1, C\sqrt{\frac{\lambda_n}{\lambda_k}}\bigg\}
  \leq 
  \min\bigg\{C,\,\sqrt{\frac{\lambda_1}{\lambda_k}}\bigg\},\qquad
  k=2,\ldots,n-1. 
 \end{align*}
 Hence, there exists $\gamma_k,\ k=2,\ldots,n-1$ such that
 \begin{align*}
  \max\bigg\{1,\, C\sqrt{\frac{\lambda_n}{\lambda_k}}\bigg\}
  ~\leq~ \gamma_k
  ~\leq~ \min\bigg\{C,\sqrt{\frac{\lambda_1}{\lambda_k}} \bigg\}. 
 \end{align*}
 Thus, $1\leq \gamma_k\leq C$ holds for all $k=2,\ldots,n-1$. 
 Moreover, $C^2\lambda_n\leq \gamma_k^2\lambda_k\leq \lambda_1$ also holds. 
 These inequalities imply $\kappa(S)=C$ and 
 $\kappa(SAS)=\lambda_1/(C^2\lambda_n)=\kappa(A)/C^2$. 
 Therefore 
 $\min_{S:\kappa(S)\leq C}\kappa(SAS)\leq \frac{\kappa(A)}{C^2}$ holds if $1\leq C^2\leq
 \kappa(A)$. 
\end{proof} 

\section{Proof of Theorem \ref{thm:min-max-evaluation}}
\label{appendix:Proof_of_min-max-theorem}
We show the proof of Theorem \ref{thm:min-max-evaluation}. 
\begin{proof}
 Let $w_1$ be the constant function taking $1$ over $\ZC$. 
 In a universal RKHS, for any $\delta>0$, there exists 
 $w\in{\cal H}$ such that $\|w_1-w\|_\infty\leq\delta$. 
 According to Appendix~D in \emcite{horn85:_matrix_analy}, 
 eigenvalues of a matrix are 
 continuous on its entries, and thus so do the minimal and maximal eigenvalues and the
 condition number as long as the condition number is well-defined. 
 Then, for any $\varepsilon>0$ and for any $\psi$ satisfying $\psi''(1)=1$, 
 there exists  $w\in{\cal H}$ such that 
 \begin{align*}
 |\kappa_0(D_{\psi,w})-\kappa_0(I_n)|\leq \varepsilon.
 \end{align*}
 Then, for fixed samples $X_1,\ldots,X_n$, we find that
\[
 \sup \{\kappa_0(D_{\psi,w})~|~w\in{\cal H}\}  ~\geq~\kappa_0(I_n). 
\]
On the other hand, for $\psi(z)=z^2/2$, we obtain
\[
\sup \{\kappa_0(D_{\psi,w})~|~w\in{\cal H}\}  ~=~\kappa_0(I_n). 
\]
Thus, \eqref{eqn:min-max-evaluation} holds. 
\end{proof}

\section{Proof of Theorem \ref{cor:cond-num-prob-convergence} }
\label{appendix:proof-of-probability-eval-theorem}
The following lemma is the key to prove Theorem \ref{cor:cond-num-prob-convergence}. 
\begin{lemma}
 \label{thm:probabilistic_evaluation}
 Suppose that the kernel function $k$ satisfies the condition in Theorem 
 \ref{cor:cond-num-prob-convergence}, and that 
 the expectation of $\psi''(\what(X_1))$ exists. 
 The probability $\Pr(\cdots)$ is defined from the distribution of samples 
 $X_1,\ldots,X_n,\,Y_1,\ldots,Y_m$. 
 Then, there exists a positive constant $\varepsilon>0$ such that 
 the probability distribution of $\kappa(H)$ is bounded above by
 \begin{align}
  \Pr\left(\kappa(H)< \delta\right)  
  ~\leq~  
  F_n\left(\frac{c}{\varepsilon}\right) + \frac{\delta}{c}
  (E[\psi''(\what(X_1))]+\lambda), 
  \label{pr-d|X-upperbound}
 \end{align}
 where $c$ is an arbitrary positive value. 
 On the other hand, 
 for any positive number $c>0$, we have 
\begin{align}
  \Pr\left(\kappa(H)>\kappa(K_{11})\big(1+\frac{c}{\lambda}\big)\right)
  ~\leq~  1-F_n(c)
  \label{pr-d|X-upperbound2}
\end{align}
 if the Gram matrix $K_{11}$ is almost surely positive definite. 
 \end{lemma} 
 
\begin{proof}
 Let $k_i$ be the $i$-th column vector of the Gram matrix $K_{11}$. 
 Due to the condition on the kernel function, 
 there exists a constant $\varepsilon>0$ such that 
 \[
 \Pr\big(\sqrt{\varepsilon} \leq (K_{11})_{ij}\leq 1,\ i,j=1,\ldots,n\big)=1, 
 \]
 where the probability is induced from the joint probability of
 $X_1,\ldots,X_n$. Hence, 
\begin{align}
 \Pr(\varepsilon n\leq \|k_i\|^2\leq n,\ i=1,\ldots,n)=1
 \label{eqn:kernel-element-bounded}
\end{align}
also holds. 
 
 Let $d_i$ be $\psi''(\what(X_i))$, then the matrix $H$ is represented as
 \[
 H=\frac{1}{n}\sum_{i=1}^{n}d_i k_i k_i^\top +\lambda K_{11} \in \Real^{n\times n}.
 \]
 Let us define 
\begin{align*}
 Y_n~=~\min_{\|a\|=1} a^\top H a,\quad 
 Z_n~=~\max_{\|a\|=1} a^\top H a. 
\end{align*}
$Y_n$ and $Z_n$ are the minimal and maximal eigenvalues of $H$. 
Thus, the condition number of $H$ is given as $\kappa(H)=Z_n/Y_n$. 

We derive an upper bound of $Y_n$ and a lower bound of $Z_n$ to prove the first
inequality \eqref{pr-d|X-upperbound}. The minimal eigenvalue is 
less than or equal to the average of all eigenvalues, and the sum of eigenvalues is equal
to the trace of the matrix. Thus, we have 
\begin{align*}
 Y_n
 ~\leq~&
 \frac{1}{n}
 \mathrm{Tr}
 \left( \frac{1}{n}\sum_{i=1}^{n}d_i k_ik_i^\top
 + \lambda K_{11}\right)
 ~\leq~
 \frac{1}{n}\sum_{i=1}^{n}d_i + \lambda, 
\end{align*}
where \eqref{eqn:kernel-element-bounded} was used. 
On the other hand, for any $j=1,\ldots,n$, the inequality 
\begin{align*}
 Z_n 
 &= \max_{\|a\|=1}\frac{1}{n}\sum_{i=1}^{n}d_i(k_i^\top a)^2+\lambda a^\top K_{11}a \\
 &\geq \max_{\|a\|=1}\frac{1}{n}\sum_{i=1}^{n}d_i(k_i^\top a)^2\\
 &\geq \frac{1}{n}\sum_{i=1}^{n}d_i(k_i^\top k_j/\|k_j\|)^2 
 \qquad\text{($k_j/\|k_j\|$ is substituted into $a$)}\\ 
 &\geq \frac{1}{n}d_i \|k_j\|^2\\ 
 &\geq\varepsilon d_j 
\end{align*} 
holds. The last inequality follows \eqref{eqn:kernel-element-bounded}. 
Hence, we have
\begin{align*}
 Z_n~\geq~ \varepsilon \max_j d_j. 
\end{align*}
Therefore, for any $\delta>0$, we have
\begin{align}
\Pr\!\left(\kappa(H)< \delta\right)
 ~\leq~
 \Pr\!\left(
 \frac{\varepsilon \max_i d_i} 
  {\frac{1}{n}\sum_{i=1}^{n}d_i+\lambda}
  <\delta \right). 
\label{upper-bound-p(k(H)<=delta)}
\end{align}
The probability of the numerator
in \eqref{upper-bound-p(k(H)<=delta)} is given as 
\begin{align*}
 \Pr(\varepsilon \max_i d_i\leq c_1) 
 ~=~ F_n\left(\frac{c_1}{\varepsilon}\right),\quad c_1>0. 
\end{align*}
For the probability of the denominator in \eqref{upper-bound-p(k(H)<=delta)}, 
we use Markov's inequality: 
\begin{align*}
\Pr\bigg(\big(\frac{1}{n}\sum_{i=1}^{n}d_i+\lambda\big)^{-1}\leq c_2\bigg)
~=~ 
 \Pr\bigg(\frac{1}{n}\sum_{i=1}^{n}d_i+\lambda\geq 1/c_2\bigg)
~\leq~ 
 c_2\left(E[d_1] +\lambda \right),\ \ \ c_2> 0. 
\end{align*}
Combining these two bounds\footnote{Let $A$, $B$, $a$, and $b$ be four positive numbers. 
If $A\geq a$ and $B\geq b$, then we have $AB\geq ab$. As the contraposition, 
if $AB< ab$, then $A< a$ or $B< b$ holds. 
}, we find
\begin{align*}
 &\Pr\!\left(
 \frac{\varepsilon \max_i d_i } 
  {\frac{1}{n}\sum_{i=1}^{n}d_i+\lambda} < c_1c_2 \right)
 ~\leq~
 F_n\left(\frac{c_1}{\varepsilon}\right) + c_2 (E[d_1]+\lambda).
\end{align*}
 Therefore, for any $\delta>0$ and $c>0$, we have \eqref{pr-d|X-upperbound}.

We prove the second inequality \eqref{pr-d|X-upperbound2}.
Let $\tau_1$ and $\tau_n$ be the maximal and minimal eigenvalues of $K_{11}$. Since all
 diagonal elements of $K_{11}$ are less than or equal to one, we have 
$0<\tau_1 \leq \mathrm{Tr}\,K_{11}\leq n$. Then, we have a lower bound of $Y_n$ and an
 upper bound of $Z_n$ as follows:
\begin{align*}
 Y_n &= \min_{\|a\|=1}\frac{1}{n}\sum_{i=1}^{n}d_i (k_i^\top a)^2 
 + \lambda a^\top K_{11} a
 ~\geq~ \lambda \tau_n,\\
 Z_n &=
 \max_{\|a\|=1}\frac{1}{n}\sum_{i=1}^{n}d_i (k_i^\top a)^2 
 + \lambda a^\top K_{11} a\\
 &\leq
 \frac{\max_j d_j}{n}
 \max_{\|a\|=1}\sum_{i=1}^{n} (k_i^\top a)^2 + \lambda \tau_1 \\
 &=
 \frac{\max_j d_j}{n} \tau_1^2 + \lambda \tau_1\\
 &\leq
 \tau_1 \max_j d_j+\lambda \tau_1, 
\end{align*}
where the last inequality for $Z_n$ follows from $\tau_1\leq n$. 
Therefore, for any $c>0$, we have 
\begin{align*}
\Pr\!\left(\kappa(H)>\kappa(K_{11})\big(1+\frac{c}{\lambda}\big)
\right)
 ~\leq~& 
 \Pr\!\left(
 \frac{\tau_1 \max_j d_j +\lambda\tau_1}{\lambda \tau_n} 
 > \kappa(K_{11})\big(1+\frac{c}{\lambda}\big) 
 \right)\\
 ~=~&  
 \Pr\!\left( \max_j d_j > c \right)\\
 ~=~&   
 1-\Pr\!\left( \max_j d_j \leq c \right)\\
 ~=~&    
 1-F_n(c). 
\end{align*} 
\end{proof}

In Lemma~\ref{thm:probabilistic_evaluation}, 
the distributions of $Y_n$ and $Z_n$ are separately
computed. This idea is borrowed from smoothed analysis of the condition numbers  
\cite{sankar06:_smoot_analy_of_condit_number}. In smoothed analysis, the probability
$\Pr(\kappa(H)\geq \delta)$ is bounded above to ensure that the condition number is
unlikely to be large. 
In the above lemma, we used the same technique also for upper-bounding the probability of
the form $\Pr(\kappa(H)\leq \delta)$. As a result, we obtained the possible lowest order
of the condition number $\kappa(H)$.  


Below, we show the proof of Theorem \ref{cor:cond-num-prob-convergence}. 
\begin{proof}
[proof of Theorem \ref{cor:cond-num-prob-convergence}]
 The inequality \eqref{pr-d|X-upperbound} in Lemma~\ref{thm:probabilistic_evaluation}
 provides 
 \begin{align*}
  \Pr(\kappa(H)< \delta_n)
  ~\leq~
  F_n\left(\frac{c_n}{\varepsilon}\right)
  +\frac{\delta_n}{c_n}(M+\lambda). 
 \end{align*}
Let $c_n$ be $\varepsilon s_n$ and $\delta_n$ be $o(s_n)$ then, 
we obtain
\begin{equation}
 \lim_{n\rightarrow\infty}\Pr(\kappa(H)<\delta_n) = 0. 
 \label{eqn:ineq-proof-1} 
\end{equation}
 We prove another inequality. 
 Due to the second inequality in Lemma~\ref{thm:probabilistic_evaluation}, 
 we have 
 \begin{align}
  \lim_{n\rightarrow\infty}
  \Pr\bigg(\kappa(H)>\kappa(K_{11})\big(1+\frac{t_n}{\lambda}\big)\bigg)
  \leq
  1- \lim_{n\rightarrow\infty}F_n(t_n)
  ~=~ 
  0
  \label{eqn:ineq-proof-2}
 \end{align}
We complete the proof by combining \eqref{eqn:ineq-proof-1} and \eqref{eqn:ineq-proof-2}. 
\end{proof}

\section{Proof of Theorem \ref{eqn:integrability}}
\label{appendix:Proof-of-reduction-theorem}
We show the proof of Theorem \ref{eqn:integrability}
\begin{proof}
 Assume that $\psi''(z)$ is not a constant function. Since $K_{11}$ is non-singular, 
 the vector $K_{11}\alpha+\frac{1}{m\lambda}K_{12}\1_m$ takes an arbitrary value in $\Real^n$
 by varying  $\alpha\in\Real^n$. Hence, each diagonal element of $D_{\psi,\alpha}$ can take
 arbitrary values in an open subset $S\subset \Real$. 
 We consider $R^{-1}M_{\psi,R}(\alpha) (R^{\top})^{-1}$ instead of $M_{\psi,R}$. 
 Suppose that there exists a matrix
 $R$ such that the matrix 
 \[
 R^{-1}M_{\psi,R}(\alpha) (R^{\top})^{-1}=\frac{1}{n}\, 
 \mathrm{diag}(s_1,\ldots,s_n) K_{11}(R^{\top})^{-1}+\lambda (R^{\top})^{-1}
 \]
 is symmetric for any $(s_1,\ldots,s_n)\in S^n$. 
 Let $a_{ij}$ be the $(i,j)$ element of $K_{11}(R^{\top})^{-1}$, and $t_{ij}$ be the
 $(i,j)$ element of $(R^\top)^{-1}$. Then, the $(i,j)$ and $(j,i)$ elements of
 $R^{-1}M_{\psi,R}(\alpha) (R^{\top})^{-1}$ are equal to 
 $\frac{1}{n}s_ia_{ij}+\lambda t_{ij}$ and $\frac{1}{n}s_ja_{ji}+\lambda t_{ji}$,
respectively.
 Due to the assumption, the equality 
 \[
 \frac{1}{n}s_ia_{ij}+\lambda t_{ij}=\frac{1}{n}s_ja_{ji}+\lambda t_{ji}
 \]
 holds for any $s_i, s_j\in S$. When $i\neq j$, we obtain $a_{ij}=a_{ji}=0$ and
 $t_{ij}=t_{ji}$. Thus, $K_{11}(R^\top)^{-1}$ should be equal to some diagonal matrix, and
 $(R^\top)^{-1}$ is a symmetric matrix. Thus, there exists a diagonal matrix 
 $Q=\mathrm{diag}(q_1,\ldots,q_n)$ such that $K_{11}=QR$ holds. 
 As a result, we have 
 $(K_{11})_{ij}=q_i R_{ij},\,(K_{11})_{ji}=q_j R_{ji},\,R_{ij}=R_{ji}$, and
 $(K_{11})_{ij}=(K_{11})_{ji}$. Hence we obtain 
 \[
 (K_{11})_{ij}~=~q_iR_{ij} ~=~q_jR_{ij}, 
 \]
 and then, $q_i=q_j$ or $R_{ij}=0$ holds for any $i$ and $j$. 
 Since $(K_{11})_{ij}$ is non-zero element, the only possibility is 
 $q_1=q_2=\cdots=q_n\neq 0$. 
 Therefore, the diagonal matrix $Q$ should be proportional to the identity
 matrix and there exists a constant $c\in\Real$ such that 
 the equality $R=cK_{11}$  holds. This equality contradicts the assumption. 
\end{proof}

\bibliographystyle{mlapa}

\input main.bbl



\end{document}

%% file: setup.tex
\newcommand{\iid}{\stackrel{i.i.d.}{\sim}}

\newcommand{\reg}{{\rm Reg}}

\def\Real{\Re}

\def\<{\langle}
\def\>{\rangle}

\def\st{\text{\rm s.\,t.\ }}

\def\Pr{{\rm Pr}}

\def\0{\mbox{\bf 0}}
\def\1{\mbox{\bf 1}}
\def\2{\mbox{\bf 2}}
\def\3{\mbox{\bf 3}}
\def\4{\mbox{\bf 4}}
\def\5{\mbox{\bf 5}}
\def\6{\mbox{\bf 6}}
\def\7{\mbox{\bf 7}}
\def\8{\mbox{\bf 8}}
\def\9{\mbox{\bf 9}}

\def\e{\mbox{\boldmath $e$}}

\def\r{\mbox{\boldmath $r$}}

\def\x{\mbox{\boldmath $x$}}

\def\GC{\mbox{$\cal G$}}
\def\HC{\mbox{$\cal H$}}

\def\ZC{\mbox{$\cal Z$}}

\def\hatw{\widehat{w}}
\def\what{\widehat{w}}